\documentclass [10pt,compsocconf]{IEEEtran}

\usepackage{algorithm,algorithmic,amsbsy,amsmath,amssymb,epsfig,bbm,mathrsfs,multirow,mathrsfs,amsthm}
\usepackage{subcaption,float}
\usepackage{graphicx}
\usepackage{xcolor}
\usepackage{bm,url}
\usepackage{varwidth}
\usepackage{enumitem,linegoal}
\usepackage{subfiles}
\usepackage{float}
\usepackage{physics} 
\usepackage{mathtools} 

\usepackage[a4paper,
            bindingoffset=0.2in,
            left=0.639in,
            right=0.639in,
            top=0.7in,
            bottom=0.594in,
            footskip=.25in]{geometry}
\usepackage{tikz}
\usepackage{tikzpagenodes}

\newtheorem{proposition}{Proposition} 
\newtheorem{theorem}{Theorem}

\newtheorem{lemma}{Lemma}
\newtheorem{remark}{\emph{Remark}}
\usepackage{soul} 

\begin{document}
	\title{Optimal Privacy Preserving for Federated Learning in Mobile Edge Computing}
    \author{\IEEEauthorblockN{Hai~M.~Nguyen, Nam~H.~Chu, Diep~N.~Nguyen, Dinh~Thai~Hoang,
    Van-Dinh Nguyen,\\
    Minh Hoàng Hà, Eryk~Dutkiewicz, and Marwan Krunz\\}
    \thanks{Hai M. Nguyen, Nam H. Chu, Diep N. Nguyen, Dinh Thai Hoang, and Eryk Dutkiewicz are with the School of Electrical and Data Engineering, University of Technology Sydney, Australia (email: hai.nguyen-2@student.uts.edu.au, namhoai.chu@student.uts.edu.au, diep.nguyen@uts.edu.au, hoang.dinh@uts.edu.au, eryk.dutkiewicz@uts.edu.au).}
    \thanks{Van-Dinh Nguyen is with the College of Engineering and Computer Science, VinUniversity, Hanoi, Vietnam (email: dinh.nv2@vinuni.edu.vn).}
    \thanks{Minh Hoàng Hà is with the ORLab, Faculty of Computer Science, Phenikaa University, Hanoi, Vietnam (email: hoang.haminh@phenikaa-uni.edu.vn).}
    \thanks{Marwan Krunz is with the Department of Electrical and Computer Engineering, University of Arizona, USA (email: krunz@email.arizona.edu).}
    \thanks{The preliminary results of this work will be presented at the IEEE International Conference on Communications (ICC), Rome, Italy, 2023 \cite{hai2023}.}
    \vspace{-10mm}
    }
    \maketitle
    \thispagestyle{empty}

\begin{abstract}
Federated Learning (FL) with quantization and deliberately added noise over wireless networks is a promising approach to preserve user differential privacy while reducing wireless resources. 
	Specifically, an FL learning process can be fused with quantized Binomial mechanism-based updates contributed by multiple users to reduce the communication overhead/cost and to protect the privacy of participating users.
	However, optimizing quantization parameters, communication resources (e.g., transmit power, bandwidth, and quantization bits), and the added noise to guarantee the differential privacy requirement and performance of the learned FL model remains an open and challenging problem.
	This article aims to jointly optimize the level of quantization, parameters of the Binomial mechanism, and communication resources so as to maximize the convergence rate under the constraints of the wireless network and differential privacy (DP) requirement.
    To that end, we first derive a novel DP budget estimation of the FL with quantization and Binomial noise that is tighter than the state-of-the-art bound.
    We then analyze the relationship between the convergence rate and the transmit power, the bandwidth, the transmission time, and the quantization/noise parameters and provide a theoretical bound on the convergence rate.
    This theoretical bound is decomposed into two components, including the variance of the global gradient and an upper bound on the quadratic bias that can be minimized by optimizing the communication resources, quantization, and added noise parameters.
	The resulting optimization turns out to be a Mixed-Integer Non-linear Programming (MINLP) problem.
	To tackle it, we first transform this MINLP problem into a new problem whose solutions are proved to be the optimal solutions of the original one.
	We then propose an approximate algorithm to solve the transformed problem with an arbitrary relative error guarantee.
	Extensive simulations show that under the same wireless resource constraints and differential privacy protection requirements, the proposed approximate algorithm achieves an accuracy close to the accuracy of the conventional FL without quantization and no added noise.
	The results can achieve a higher convergence rate while preserving users' privacy.
\end{abstract}

\begin{IEEEkeywords}
	Binomial mechanism, differential privacy, federated learning, quantization, communication resources, convergence rate, approximate algorithm, wireless.
\end{IEEEkeywords}

\section{Introduction}
The rapid growth of mobile devices and services resulted in a huge amount of data for artificial intelligence (AI) based mobile applications, e.g., healthcare and e-commerce services. 
	However, constructing a global model from big data is still challenging.
	First, due to privacy concerns, mobile users are not always willing to share their raw data (e.g., location, information, and travel habits/data).
	Second, fusing users' data at a server may incur significant communication overhead/cost.
	In this context, Federated Learning (FL), among various distributed learning frameworks, has recently emerged as a potential solution to address these two challenges. 
	Specifically, instead of requiring mobile users to share their raw data, FL only requires users to send their gradients based on their local data to a centralized server for the learning process. 
	By doing so, not only the communication cost significantly decreases but also users' privacy concerns are alleviated \cite{Agarwal2018}.

However, FL faces different challenges when deployed over wireless networks \cite{chen2021distributed}.
	First, although only local gradients from mobile users are sent to the server, the communication cost remains a major concern for the FL over wireless networks (FLoWNs).
	The reason is that a mobile AI-based application may require updates/data from a large number of devices (in the order of thousands or more), thus putting significant stress on network resources \cite{lin2017deep}.
	Additionally, to achieve a certain accuracy level, multiple rounds of information exchange between the participating devices and the aggregating server are required.
	These problems are particularly more pronounced with complex deep learning models in which a local update may contain millions of parameters \cite{He2016}.
	Second, due to its broadcast/open nature, wireless networks are vulnerable to many types of attacks, such as Man-in-the-Middle, DDoS, and Sybil, leading to privacy concerns \cite{zou16}. 
	Recent studies (e.g., \cite{nasr2018comprehensive}, \cite{DBLP:journals/corr/abs-2001-02610}, and \cite{zhu19}) revealed that it is possible to retrieve the original data from the victims' shared local gradients.
	This can void the privacy protection advantage of FL.

To address the above challenges, a few works have adopted a quantization technique to reduce communication costs. 
    For example, in \cite{seide20141} and \cite{strom2015scalable}, the gradient elements are rounded to either 1 or -1. 
    The authors in \cite{Reisizadeh20a} presented another quantized FL framework that periodically averages the model's parameters at the server's side and quantizes the message-passing from edge nodes to the server.
    Furthermore, to improve the performance of FL (e.g., convergence rate), each node updates its local model by applying stochastic gradient descent (SGD) after a fixed number of iterations.
    Finally, to better scale the system, the server only updates the model with a fraction of the total nodes in each round. 
    Similarly, the study in \cite{9054168} proposed an FL with quantization constraint on the gradients.
    To establish a theoretical guarantee, the authors showed that the error caused by the quantization scheme is bounded by a term that decreases exponentially with the number of users.
    In \cite{pmlr-v130-haddadpour21a}, the authors proposed algorithms with periodic quantization and analyzed their convergence properties.
    In particular, they derived an upper bound on the learning time of various objective functions, including strongly convex and non-convex ones.
    In \cite{9712310}, the authors proposed a heterogeneous quantization approach that allows users to adapt the quantization parameters according to their communication resources.
    The network was partitioned into groups, and the local model user updates were divided into segments and aggregated the updates on segments.
    The authors demonstrated that their framework guarantees secure aggregation simultaneously in Byzantine scenarios and achieves convergence in non-Byzantine scenarios.
    Unlike the above works, the authors in \cite{9413697} proposed a strategy to adjust the quantization levels during the training process.
    Through simulations on deep neural networks, they showed that their method achieves fewer communicated bits compared to a fixed quantization level policy.
    Interested readers are referred to \cite{survey_quantization_rokh} for a comprehensive survey on model quantization for deep neural networks.

To address privacy concerns in FL, there is a rising interest in Differential Privacy (DP), which is a scheme to share group pattern information of a dataset while securing the privacy of individuals.
    It uses a privacy budget parameter to measure the distinguishing probability between two datasets that differ by one individual record.
    The idea of DP is to add noise to private records in the dataset before aggregation.
    In \cite{9069945}, the authors proved that by adjusting the artificial noise, any privacy protection level can be satisfied.
    They also analyzed the optimal number of devices to maximize the convergence rate of the underlying learning process.
    This theoretical analysis also captures the trade-off between the privacy level and the convergence rate as well as the impact of the number of devices.
    To preserve user privacy and reduce the communication cost, the authors in \cite{9253545} integrated FL with two-bit quantization and local DP mechanisms over an Internet of Vehicles network.
    The local DP mechanisms include a three-output mechanism for a small privacy budget, an optimal piecewise mechanism (PM-OPT), a suboptimal mechanism (PM-SUB), and a hybrid combining of PM-OPT and PM-SUB mechanism for a large privacy budget. 
    In \cite{9413764}, the authors considered a Gaussian mechanism for adding noise to the gradients of FL.
    Compared to other works, this work achieves a tighter bound on the privacy budget.
    The authors in \cite{pmlr-v130-girgis21a} studied the communication efficiency, privacy, and convergence trade-offs between the federated communication cost and local DP SDG algorithm.
    In particular, the proposed algorithm applied for the empirical risk minimization (ERM) optimization problem while guaranteeing the communication efficiency and privacy restrictions applied for the FL network.
    The authors leveraged the advantages of client subsampling and data subsampling as well as the shuffled model of privacy to deal with the limitation of DP.
    Theoretically, the proposed algorithm provides a lower bound on the ERM problem.
    Readers are referred to \cite{el2022differential} for a more comprehensive survey of DP in FL.

    Note that all aforementioned works do not take into account optimization of the system parameters (e.g., transmit power, bandwidth, and transmission time) and the quantization/noise factors while guaranteeing the DP of users in the underlying FL process.
	This problem is, in fact, very challenging since privacy-preserving methods often add noise to data or use quantization, hence significantly reducing the learning quality.
	For example, the authors in \cite{Agarwal2018} proposed a framework that leverages quantization and Binomial mechanisms to reduce communication costs and provide DP.
	However, they only focused on the theoretical side and did not study the inherent factors of an FL system over wireless networks, e.g., limited bandwidth, transmit power, and transmission time.
    Moreover, the optimization of wireless/communication resources is often done on a short-term basis (e.g., at a packet of a frame length) while the learning convergence rate/accuracy must be optimized over a much longer time scale.
	Studying the impact of these system parameters on the performance of FLoWNs with regard to DP protection, learning accuracy, and convergence rate is the focus of this article. Our major contributions are as follows:
\begin{itemize}
    \item Derive a novel differential privacy budget estimation of the FL with quantization and Binomial noise.
    This tighter privacy budget estimation allows us to study the convergence rate optimization problem over a larger feasible region, hence achieving a higher convergence rate compared to ones that use the known privacy budget estimation \cite{Agarwal2018}. This is inline with the idea recently reported in \cite{jingliang-yuan-infocom-2023} where the privacy budget can be treated as a type of resource.
    
    \item Analyze the relationship between the convergence rate and the transmit power, the bandwidth, the transmission time, and the quantization/noise parameters and provide a theoretical bound on the convergence rate.
    Later, it can be seen that the bandwidth, the transmission time, and the transmit power only appear on the right-hand side of a constraint which is a function that monotonically increases with respect to these parameters.
    Thus, we can fix two among three parameters, i.e., the bandwidth, the transmission time, and the transmit power, to optimize the other parameter.
    \item Decompose the bound into two components, including the variance of the global gradient and the quadratic bias introduced by the quantization/noise mechanism, that can be minimized by optimizing the transmit power, quantization, and noise-added parameters.
The resulting optimization turns out to be a Mixed-Integer Non-linear Programming (MINLP) problem.
To tackle it, we transform this MINLP problem into a new problem whose solutions are proved to be the optimal solutions of the original one.
We then design an approximate algorithm that can solve the transformed problem with an arbitrary relative error guarantee. 
\item Extensive simulations show that for the same resources, the proposed approach achieves an accuracy close to that of the conventional FL without quantization and no noise added.
This suggests a faster convergence rate for the proposed wireless FL framework while optimally preserving users' privacy.
    
\end{itemize}
    
	The remainder of this paper is organized as follows. 
	Section \ref{s:system-model} presents the architecture of FL with added noise and quantized gradients over Mobile Edge Computing (MEC), the theoretical analysis of the privacy budget estimation, and the bound on the convergence rate.
    The problem formulation, its approximate algorithm, and the complexity analysis are in Section \ref{s:problem-formulation}. The experiments and discussion of the numerical results are in Section \ref{s:numerical-results}.
	Finally, conclusions are drawn in Section \ref{s:conclusions}.

\section{System Model and convergence analysis} \label{s:system-model}
This work considers a MEC architecture in which a Mobile Edge Server (MES) orchestrates an FL process 
consisting of $M$ mobile devices
\cite{7474412}.
	 Each mobile device
  $k$, $k \in \{1, \ldots, M\}$, has a private local dataset.
	 This dataset can be created through the user's activities captured by this device (e.g., health- or travel-related data) and hence subject to data privacy protection.
	    \begin{figure}[H]
		\centerline{\includegraphics[width=0.475\textwidth]{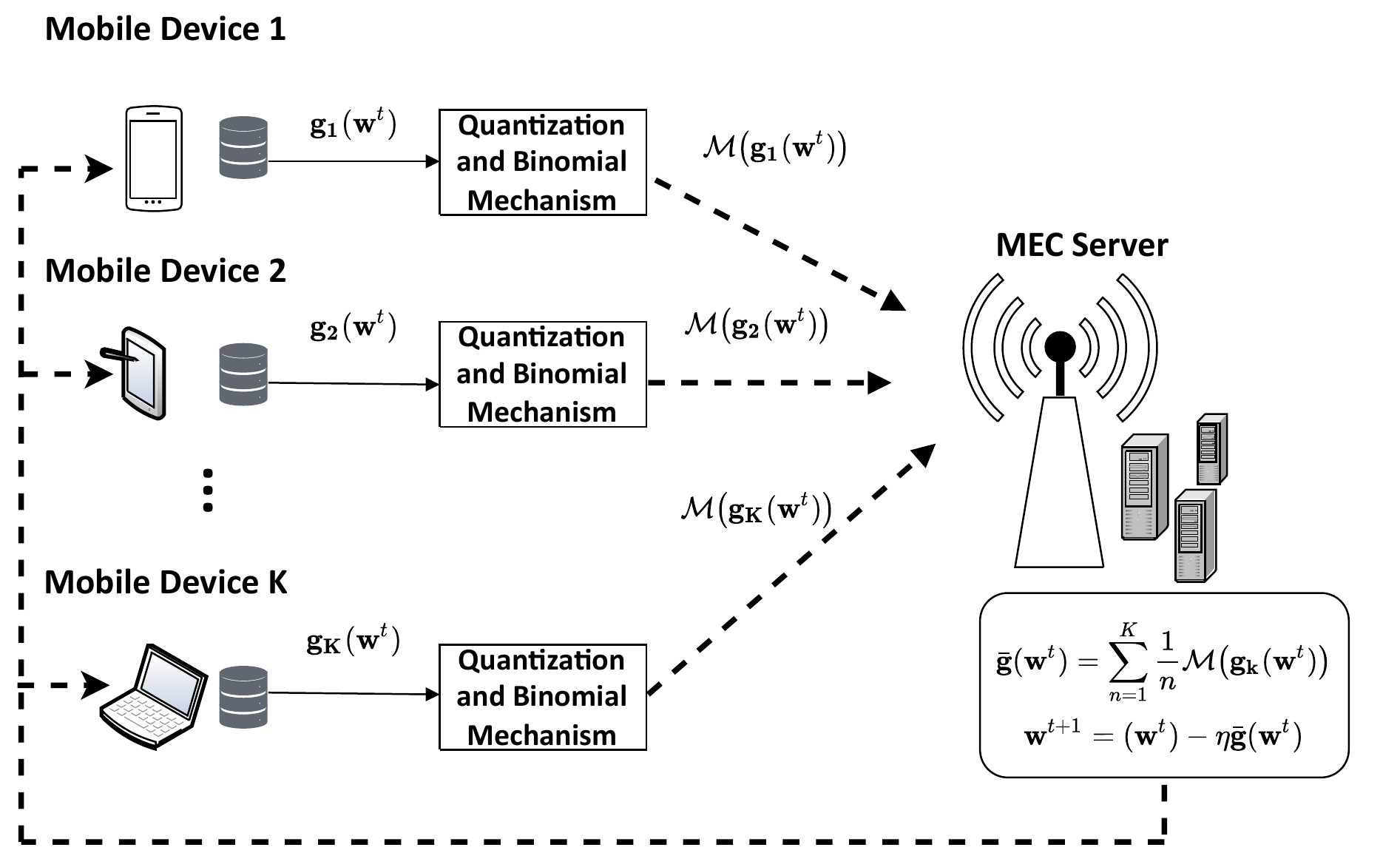}}
		\caption{Stochastic Binomial mechanism and Level quantization FL over MEC architecture.}
		\label{fig:1}
	\end{figure}
\subsection{Federated Learning over MEC}\label{com-efficiency-DP}
In general, the tasks in FL can be expressed as an optimization problem \cite{mcmahan17} of the average loss function $F(\mathbf{w})$:
	\begin{equation}
 \min_{\mathbf{\mathbf{w}} \in \mathbb{R}^d} \left\{ F(\mathbf{w}) = \frac{1}{M} \sum_{k = 1}^{M} f_k(\mathbf{w})\right\},
    \label{eq1}
	\end{equation}
	where  $f_k(\mathbf{w})$ is the loss function processed on device $k \in \{1, \ldots, M\}$, $\mathbf{w} \in \mathbb{R}^d$ is the gradient vector and $d$ is the dimension of $\mathbf{w}$.
	The objective is to minimize the loss function by finding the optimal model parameter set $\mathbf{w}$.
	Problem \eqref{eq1} can be solved by the Federated Stochastic Gradient Descent (FSGD) approach \cite{mcmahan17}, which continuously iterates the following steps:
	\begin{enumerate}
		\item \textbf{\textit{Broadcast:}} At the beginning of iteration $t$, the MES broadcasts the current model parameters $\mathbf{w}^t$ to all $M$ mobile devices.
		\item \textbf{\textit{Local computation:}} After receiving $\mathbf{w}^t$, mobile device $k$ computes its local gradient $\mathbf{g}_k(\mathbf{w}^t)=\nabla f_k(\mathbf{w}^t)$ based on its local dataset, and then sends $\mathbf{g}_k(\mathbf{w}^t)$ to the server. 
		\item \textbf{\textit{Model update:}} 
  The MES selects a subset of mobile devices $\mathcal{K}\subset \{1, \ldots, M\}$ to update the global gradients. As soon as it receives the updates from all mobile devices in $\mathcal{K}$, the MES estimates the gradient $\nabla F(\mathbf{w}^t)$ by aggregating the local gradients.
    Then, it updates the model parameters for the next iteration $\mathbf{w}^{t + 1}$ as follows:
		\begin{align*}
    		\mathbf{g}(\mathbf{w}^t) =	\frac{1}{K} \sum_{k \in \mathcal{K}} \mathbf{g}_{k}(\mathbf{w}^t), \\ \mathbf{w}^{t + 1} = \mathbf{w}^t - \gamma \mathbf{g}(\mathbf{w}^t),
		\end{align*}
	   \end{enumerate}
    where $K = |\mathcal{K}|$ is the number of selected devices and $\gamma$ is the learning rate. 
	
	Since the expectation of the gradient $\mathbb{E}[\mathbf{g}(\mathbf{w}^t)] = \nabla F(\mathbf{w}^t)$  \cite{Ghadimi2013}, $\mathbf{g}(\mathbf{w}^t)$ is an unbiased estimation of $\nabla F(\mathbf{w}^t)$. The process stops when the loss function converges, e.g., $||\mathbf{g}(\mathbf{w}^t)||^2 \le \theta$ where $ \theta$ is a given threshold  $0 \le \theta \le 1$, or achieves a desirable accuracy. 
	In the next section, we describe how the quantization and Binomial mechanisms can effectively lower communication costs and guarantee data privacy for the FL.
	
\subsection{Quantization and Privacy for FL over MEC}
As mentioned before, to deal with data-intensive local gradients (millions of data points, as in \cite{zaremba2014}) that significantly consume the resources of mobile devices and the MES, quantization is often employed  \cite{Agarwal2018, Reisizadeh20a}.
    Therefore, we adopt a stochastic $q$-level quantization which converts the real values of the gradients into integer values with $\log_2(q)$ bits \cite{horvath2022stochastic}, thus significantly reducing the communication overhead.
    This quantization mechanism is parameterized by the maximum value of the gradient $D$ and the quantization level $q$.
	 	
At the beginning of the training process, the server instructs the devices of the quantization parameters $D$ and $q$.
	 A simple choice of $D$ is the maximum value of the loss function gradient \cite{Agarwal2018}.
	 Then, all devices re-scale each element $g_k^i(\mathbf{w}^t)$ of their local gradients $\mathbf{g}_k(\mathbf{w}^t)$ to the range $\left[-D, \ D\right]$ \cite{Agarwal2018}, where $i$ is the index of element $g^i_k(\mathbf{w}^t)$ of the vector $\mathbf{g}_k(\mathbf{w}^t)$.
	 Specifically, similar to \cite{Agarwal2018}, we define $V(j)$ of an integer $j \in [0, q - 1]$ as follows:
	\begin{equation}
		V(j) = -D+\frac{2D}{q-1}j. \label{eq:18}
	\end{equation}
    Thus, $V(j)$ is always within $[-D, D]$.
	Then, the quantized local gradient of $g_k^i(\mathbf{w}^t)$, denoted by $Q(g_k^i(\mathbf{w}^t))$, is defined as follows:
	\begin{equation*}
		Q(g_k^i(\mathbf{w}^t)) =
		\begin{cases}
			V(r+1) & \text{with probability }  \frac{g_k^i(\mathbf{w}^t) - V(r)}{V(r+1) - V(r)}, \\
			V(r) & \text{otherwise},
		\end{cases}			
	\end{equation*}
	where $r \in [0, \ q - 1]$ is an integer such that the value of $g_k^i(\mathbf{w}^t)$ is within $\big[V(r)$,  $V(r+1)\big)$.
    Thereby, the gradient size is significantly reduced by controlling the parameters, i.e., the quantization level $q$ and the maximum value of the gradient $D$.  It is worth noting that here we assume homogeneous users/devices where the same quantization parameters are applied to gradients of all users and all rounds. In practice, one can also use adaptive gradient quantization \cite{liu2022communicationefficient} to leverage the heterogeneity of users/devices or even the change in each updating round/iteration.   
		
Another major challenge when employing an FL system over MEC is the leak of users' privacy while devices share their local gradients.
	A potential solution to guarantee the DP for mobile devices is to add random noise to the local gradient updates \cite{Dwork2014}.
	As defined in \cite{Dwork2014}, a randomized mechanism $\mathcal{M}$ satisfies $(\epsilon,\delta)$-differential privacy if for two neighboring input datasets, say $x$ and $y$, differ by up to one element, and for any output set $\mathcal{S}$ of $\mathcal{M}$ we have:
		\[\Pr{\left[\mathcal{M}\left(x\right)\in\mathcal{S}\right]}\le e^{\epsilon}\Pr{\left[\mathcal{M}\left(y\right)\in\mathcal{S}\right]}+\delta,\]
	where $\epsilon > 0$ is a parameter that represents the privacy loss, often referred to as the differential privacy budget. 
	The smaller the value of $\epsilon$ is, the better privacy protection can be achieved.
    The given $\delta$ is an upper bound on the probability of a bad event, i.e., the privacy is broken.
	In our work, we adopt the Binomial mechanism \cite{Agarwal2018} to achieve the $(\epsilon, \delta)$-differential privacy.
				
Under the Binomial mechanism, the noise vector $\mathbf{z}$ is drawn from the Binomial distribution $\mathcal{B}\left(n,p\right)$, i.e., for each coordinate $i$, $z_i \sim \mathcal{B}\left(n, p\right)$ is added to $Q(g^i_k(\mathbf{w}^t))$ as \cite{10.1007/11761679_29}:
		\[\mathcal{M}(\mathbf{g}_k(\mathbf{w}^t)) = Q\left(\mathbf{g}_k(\mathbf{w}^t)\right) +s\left(\mathbf{z}-np\right),\]
	where $n$ and $p$ are parameters of the Binomial distribution, and $s$ is the noise scale, computed as in \cite{Agarwal2018}: \begin{equation} \label{eq:noise-level}
	    s = \frac{2D}{q-1}.
	\end{equation}
	
The above stochastic level quantization and Binomial mechanism-based FL (referred to as SLQBM-FL)  \cite{Agarwal2018} under the mobile edge computing framework is illustrated in Fig. \ref{fig:1}.
    With $K$ selected mobile devices, the SLQBM-FL is proved to achieve $\left(\epsilon,\delta\right)$-differential privacy \cite{Agarwal2018} if the following inequality holds:
	\begin{equation}\label{eq:DP_cons}
		Knp(1 - p) \ge \max \bigg\{23\ln\frac{10d}{\delta}, 2(q + 1)\bigg\},
	\end{equation}
	then $\epsilon$ is calculated as:
	\begin{alignat}{3}
	 \nonumber \epsilon & = &\frac{\Delta_2 \sqrt{2 \ln\frac{1.25}{\delta}}}{\sqrt{np(1 - p)}}
	+ \frac{\Delta_2 c_p \sqrt{\ln\frac{10}{\delta}} + \Delta_1 b_p} {np(1 - p)(1-\frac{\delta}{10})} \\
	&& + \frac{\frac{2}{3}\Delta_\infty \ln\frac{1.25}{\delta} + \Delta_\infty d_p \ln\frac{20d}{\delta} \ln\frac{10}{\delta}} {np(1 - p)}, \label{eq:expression-epsilon}
	\end{alignat}
	where
	\begin{alignat}{3}
	&c_p &&\overset{\Delta}{=} \sqrt2\left(3p^3+3\left(1-p\right)^3+2p^2+2\left(1-p\right)^2\right) \label{eq:cp-exp},\\
	&b_p &&\overset{\Delta}{=} \frac{2}{3}\left(p^2+\left(1-p\right)^2\right)+\left(1-2p\right) \label{eq:bp-exp},\\
	&d_p &&\overset{\Delta}{=} \frac{4}{3}\left(p^2+\left(1-p\right)^2\right) \label{eq:dp-exp},
	\end{alignat}
and the sensitivity bounds $\Delta_1$, $\Delta_2$, and $\Delta_\infty$:
	\begin{alignat}{2}
	&\Delta_1  &&\overset{\Delta}{=} \frac{2\sqrt{d}D}{s} + \sqrt{\frac{4\sqrt{d}D\ln\frac{2}{\delta}}{s}} + \frac{4}{3}\ln\frac{2}{\delta} \label{eq:Delta_1-1-exp},\\
	&\Delta_2  &&\overset{\Delta}{=} \frac{2D}{s} + \sqrt{\Delta_1 +  \sqrt{\frac{4\sqrt{d}D\ln\frac{2}{\delta}}{s}}} \label{eq:Delta_2-1-exp}, \\
	&\Delta_\infty &&\overset{\Delta}{=} q + 1. 
    \label{eq:Delta_inf-1-exp}
	\end{alignat}
The rationale behind Eq. \eqref{eq:DP_cons} is that the variance of the Binomial mechanism $np(1-p)$ needs to exceed a lower bound to guarantee the ($\epsilon, \delta$)-differential privacy requirement.
    This lower bound is directly proportional to the number of dimensions $d$ and inversely proportional to the probability of privacy broken.

The authors in \cite{Agarwal2018} showed that the right-hand side of Eq. \eqref{eq:expression-epsilon} is a privacy budget estimation when the condition \eqref{eq:DP_cons} holds. Here, we derive a tighter privacy budget estimation in Theorem \ref{therm-new-estimation-privacy-budget} below. Intuitively, for the same value of $(q, n, p, P_k)$, the privacy budget estimation in \eqref{new-privacy-budget-estimation} is smaller than the privacy budget estimation in \eqref{eq:expression-epsilon} (more details are presented in Appendix \ref{new-estimation-privacy-budget-thrm-proof}).
Therefore, given a privacy budget upper bound $\bar{\epsilon}$, the proposed privacy budget estimation allows us to study the convergence rate optimization problem over a larger feasible region, compared to the known privacy budget estimation in [1] (hence potentially yielding a higher convergence rate). The significance of this tighter privacy budget estimation can be more pronounced where the privacy budget can be treated as a type of resource, as recently reported in \cite{jingliang-yuan-infocom-2023}. 

    
\begin{theorem}[Privacy Budget Estimation]\label{therm-new-estimation-privacy-budget}
For any $\delta$ that satisfies:
\begin{equation*}
    Knp(1 - p) \ge \max \bigg\{23\ln\frac{10d}{\delta}, 2(q + 1)\bigg\},
\end{equation*}
the Binomial mechanism is ($\epsilon, \delta$)-privacy for
\begin{align}
    \epsilon = & \frac{\Delta_2\sqrt{2\ln{\frac{1.25}{\delta}}}}{\sqrt{np(1-p)}} + \alpha \frac{\Delta_1 (np(1-p) + 1)}{n^2p^2(1-p)^2(1-\frac{\delta}{10})} \nonumber
\end{align}
\begin{align}
    & (p^2+(1-p)^2) + \frac{\Delta_2}{\sqrt{1 - \frac{\delta}{10}}} \sqrt{2S_1\ln{\frac{10}{\delta}}} \nonumber
\end{align}
\begin{align}
    & + \frac{2}{3}\alpha\frac{S_2\left(p^2 + (1-p)^2\right)\ln{\frac{10}{\delta}} \Delta_{\infty}}{n^2p^2(1-p)^2} \nonumber\\
    & + \frac{2\ln{\frac{1.25}{\delta}}\Delta_{\infty}}{np(1-p)} \label{new-privacy-budget-estimation}
\end{align}
where:
\begin{alignat}{3}
    &\alpha &=& -3 - 9\ln{\frac{2}{3}},\nonumber\\
    &S_1 &=&   
        \frac{3p^2-3p+1}{n(n+1)(n+2)p^2(1-p)^2}\bigg[ 3n+2  \nonumber\\
        &&& + \frac{2}{p(1-p)} \bigg],\label{s-1-formula}
\end{alignat}
\begin{alignat}{3}
    &S_2 &=& \bigg( \sqrt{2np(1-p)\ln{\frac{20d}{\delta}}} + 1 + \frac{2}{3}\max\{p, 1-p\} \nonumber\\
    &&&  \ln{\frac{20d}{\delta}} \bigg) ^2,
    \label{s-2-formula}
\end{alignat}
$\Delta_1$, $\Delta_2$, and $\Delta_\infty$ are defined in Eqs \eqref{eq:Delta_1-1-exp}-\eqref{eq:Delta_inf-1-exp}. The proposed privacy budget estimation is tighter than the privacy budget estimation in \cite{Agarwal2018}, i.e., the right-hand side value of Eq. \eqref{new-privacy-budget-estimation} is smaller than that of Eq. \eqref{eq:expression-epsilon} with respect to the same system parameters' and quantization/noise parameters' values.
\end{theorem}
\begin{proof}
    See Appendix \ref{new-estimation-privacy-budget-thrm-proof}.
\end{proof}

\begin{proposition} \label{epsilon-properties-remark} Unlike the privacy budget estimation derived in \cite{Agarwal2018}, our privacy budget estimation in Eq. \eqref{new-privacy-budget-estimation} satisfies the following properties that are helpful in effectively designing the approximate algorithm to maximize the convergence rate under wireless networks' resource constraints and quantization/noise constraints in Section \ref{alg-solve-p2}:
(i) Privacy budget estimation is symmetric with respect to the Binomial mechanism probability  $p$;
(ii) Privacy budget estimation monotonically decreases with respect to the Binomial mechanism trial number $n$; and
(iii) Privacy budget estimation monotonically increases with respect to the level quantization $q$.
\end{proposition}
\begin{proof}
    See Appendix \ref{remark-epsilon-est-pro}.
\end{proof}

As later seen in Section \ref{alg-solve-p2}, thanks to the symmetric and monotonic properties of privacy budget estimation in Eq. \eqref{new-privacy-budget-estimation}, we design an effective algorithm to solve the proposed convergence rate optimization problem under network resource and DP constraints.
First, due to the symmetric property, we only need to explore half of the domain set of $p$, i.e., $(0,1/2]$ or $[1/2, 1)$ instead of the whole interval $(0, 1)$.
Second, the monotonicity of privacy budget estimation with respect to either $n$ or $q$ inspires a binary search over one of these variables to solve the problem.

After the quantization and Binomial processes, instead of sending the actual gradient as in a conventional FL, each device $k$ sends its quantized and added-noise $\mathcal{M}(\mathbf{g}_k(\mathbf{w}^t))$ to the MES.
	The quantization significantly reduces the size of the gradient.
	In particular, the reduced size of the local quantized and noise-added gradient is $d\log_2(q+n)$ bits \cite{Agarwal2018}.
	We assume that Orthogonal Frequency-Division Multiple Access (OFDMA) is employed for the uplink between mobile devices and the MES.
	Without losing the generality, here we assume all devices have the same bandwidth $W$ and transmission time $T$, and the server uses a dedicated channel to broadcast global updates to all devices.
	It should be noted that the size of quantized and noise-added gradient must not exceed the capacity of its channel:
	\begin{equation}
	    d\log_2(q + n) \leq R_k T, \mbox{ } k \in \mathcal{K},
	\label{quantized-size-cons-1st-form} \end{equation}
	where $R_k$ is the transmission rate of device $k$ given by Shannon's equation
	\begin{equation}
		R_k= W\log_2\bigg(1 + \frac{P_k h_k}{\omega_0}\bigg), \mbox{ } k \in \mathcal{K},
	\end{equation}
	where $\omega_0$, $h_k$, and $P_k$ are the noise power, the channel gain, and the transmit power of device $k$, respectively. 
	Thus \eqref{quantized-size-cons-1st-form} can be re-expressed as
	\begin{equation}\label{eq:channel_cap_cons}
		d\log_2(q+n) \le WT\log_2\bigg(1+ \frac{P_k h_k}{\omega_0}\bigg), \mbox{ } k \in \mathcal{K}.
	\end{equation}
Later, we can see that the wireless resource parameters, including the transmit power, the bandwidth, and the transmission time, only appear on the right-hand side of constraint \eqref{eq:channel_cap_cons}.
In addition, the right-hand-side function of \eqref{eq:channel_cap_cons} monotonically increases with respect to these three resource parameters.
Hence, we can optimize one of these parameters while fixing the other two parameters at the maximum allowed values to optimize the convergence rate.
In particular, here we vary the transmit power $P_k \in [P_k^{\min}, P_k^{\max}]$ and fix the bandwidth and transmission time at $W$ and $T$, where $W$ and $T$ are the maximum bandwidth and transmission time, respectively.

Finally, the server aggregates the received quantized and randomized gradients in a similar way to the conventional SGD:
	\begin{equation}
		\tilde{\mathbf{g}}(\mathbf{w}^t) = \frac{1}{K} \sum_{k \in \mathcal{K}} \mathcal{M}\left(\mathbf{g}_k(\mathbf{w}^t)\right).
		\label{eq:6}
	\end{equation}
    The learning process continues until it converges.
    In the next section, we present a convergence analysis of our proposed SLQBM-FL.
	
\subsection{Convergence Rate Analysis of SLQBM-FL}
In this section, we analyze how the quantization and Binomial mechanisms affect the convergence rate of the FL system.
	When using the SGD to solve problem \eqref{eq1}, it is well understood that the algorithm achieves an accuracy $\theta$ after $\mathcal{O} (1/\ln(\theta))$ iterations \cite{Ma2015}.
	However, the convergence rate under the estimation of the global gradient at the server is still unknown. 
	To derive the convergence rate of SLQBM-FL, as in \cite{Agarwal2018}, \cite{Ghadimi2013}, we assume the following conditions hold: 
	 \begin{itemize}
	 	\item The loss function $F(\mathbf{w}^t)$ is $L$-smooth: \[||\nabla F(\mathbf{x}) - \nabla F(\mathbf{y})|| \leq L||\mathbf{x} - \mathbf{y}||.\]
	 	\item The gradient element of the loss function has an upper bound: \[|(\nabla F)^j(\mathbf{w}^t)| \le G, \mbox{ } j \in \{1, \ldots, d\}.\]
	 	\item The gap between the values of the loss function at an initial parameter $\mathbf{w}^0$ and at an optimal parameter $\mathbf{w}^{*}$ is bounded \[F(\mathbf{w}^0)- F(\mathbf{w}^{*}) \le G_f.\]
	 \end{itemize}
	 
Following a similar approach as in \cite{Ghadimi2013}, which shows the convergence rate of their Randomized SGD algorithm for computing an $(\theta, \Lambda)$-solution, i.e., a point $\tilde{\mathbf{w}}$ such that $\mbox{Pr}(||\nabla F(\tilde{\mathbf{w}}||^2 \leq \theta) \geq 1 - \Lambda$ for $\theta > 0$ and $\Lambda \in (0, 1)$, we formally state the convergence rate of SLQBM-FL in Theorem \ref{theorem:1}. Theorem \ref{upperbound-convergence-rate} states the upper bounds on the factors that control this convergence rate.
	\begin{theorem}
	[Convergence Rate of SLQBM-FL]
		The number of iterations performed by SLQBM-FL to achieve an $(\theta, \Lambda)$-solution, for $\theta > 0$ and $\Lambda \in (0,1)$, is bounded by:
		\[\mathcal{O} \left\{\frac{1}{\Lambda\theta} +\frac{\sigma^2}{\Lambda^2\theta^2} \right\}, \]
	 where $\sigma^2 = U + B$ with $U$ being the variance of the global gradient, and $B$ being the quadratic bias introduced by $\mathcal{M}$:
	    \begin{alignat}{3}
		U &= \max_{1 \le t \le T} \mathbb{E}\left[\Vert \mathbf{g}(\mathbf{w}^t) - \nabla F(\mathbf{w}^t)\Vert^2 \right], \label{eq:9}\\
        B &= \max_{1 \le t \le T} \mathbb{E}_{\mathcal{M}}\left[\Vert \mathbf{g}(\mathbf{w}^t) - \tilde{\mathbf{g}}(\mathbf{w}^t) \Vert^2 \right].	\label{eq:8}
	    \end{alignat}
		\label{theorem:1}
	\end{theorem}
	\begin{proof}
		See Appendix \ref{app:c}.
	\end{proof}

Note that the authors of \cite{Agarwal2018} parameterized the convergence rate through the factor $\sigma'$:
$$\sigma'^2 = c_1(U + B) + c_2 \max_{1\leq t \leq T}||\mathbb{E}[\mathbf{g}(\mathbf{w}^t) - \tilde{\mathbf{g}}(\mathbf{w}^t)]||$$
where $c_1$ and $c_2$ are constants.
The advantage of our factor $\sigma^2$ compared to above $\sigma'^2$ is that we only need to consider $U$ and $B$ and do not need to consider $\max_{1\leq t \leq T}||\mathbb{E}[\mathbf{g}(\mathbf{w}^t) - \tilde{\mathbf{g}}(\mathbf{w}^t)]||$.

\begin{theorem}
    [Upper Bounds on Variances of Global Gradient and Quadratic Bias]\label{upperbound-convergence-rate}
    The variance of the global gradient and the quadratic bias introduced by $\mathcal{M}$ are bounded as follows:
    \begin{alignat}{3}
    &U\leq& 4\bigg(\frac{M-K}{M}\bigg)^2dG^2, \label{U-upperbound}\
    \end{alignat}
    \begin{alignat}{3}
    \frac{4dG^2np(1-p)}{K(q-1)^2} \leq &B\leq& \frac{4dG^2\big(1+np(1-p)\big)}{K(q-1)^2}. 	\label{B-upperbound}
    \end{alignat}
    If the gradient elements $\nabla f_k^j(\mathbf{w}^t)$ are independent and identically distributed (i.i.d.) for all devices, we have the following bound:
    \begin{alignat}{3}
    &U\leq& \frac{8(M-K)}{M^2}\frac{dG^2}{K}.
    \label{assumped-U-upperbound}\
    \end{alignat}
\end{theorem}
\begin{proof}
    See Appendix \ref{U-B-upperbound-proof}.
\end{proof}

From Theorem \ref{theorem:1}, the convergence rate of SLQBM-FL is controlled by $\sigma$.
	Since $\sigma^2 = 2U + 2B$, reducing $U$ and $B$ will speed up the learning process.
    Theorem \ref{upperbound-convergence-rate} and its following observations suggest that $U$ is insignificant in comparison to $B$.
    We therefore only need to minimize $B$ to practically improve the convergence rate.
    First, from \eqref{U-upperbound}, we observe that when $K$ is close to $M$, $U$ is very marginal or even equal to zero when $K=M$.
    Second, for $M>>K$ (that is often the case in practice), the inequality \eqref{assumped-U-upperbound}\footnote{Theoretically, the gradients of the loss function $g_k^j(\cdot)$ are aggregated to update the model, and therefore not independent.
    However, in our case, we quantize and add random noise to each gradient element.
    Furthermore, in the FL framework where $K$ devices are chosen randomly and $M>>K$, the correlation between $\nabla f^j_k(\mathbf{w}^t)$ and $\nabla f^j_{k'}(\mathbf{w}^t)$ is small in comparison to $G$, for $k \neq k'$.
    Thus, we can practically assume that the gradient elements $\nabla f^j_k(\mathbf{w}^t)$ are i.i.d.}
    also suggests that $U$ tends to zero for a large number of devices $M$.
    From \eqref{B-upperbound}, we found that $np(1-p) >> (q-1)^2$, we thus have $B >> U$.
    Third, we observe from \eqref{B-upperbound} that the ratio of the upper bound and the lower bound of $B$ is approximately 1 as $np(1-p)$ is large.
    Therefore, we can optimize the upper bound of $B$ which is a function of the wireless resources, quantization and noise-added parameters\footnote{Although, the wireless resource parameters, i.e., the bandwidth, the transmit power, and the transmission time do not appear in the formula \eqref{B-upperbound}, they implicitly impact on $B$ through constraints, e.g., \eqref{eq:DP_cons}.} to improve the learning rate.

\section{Problem formulation and solutions} \label{s:problem-formulation}
As analyzed in the previous section, to maximize the convergence rate of SLQBM-FL under the constraints on the network resources and $(\epsilon, \delta)$-differential privacy protection, we can minimize the upper bound on $B$ in Eq. \eqref{B-upperbound} by jointly optimizing the transmit power, the quantization level, and the parameters of Binomial mechanism.
    The optimization problem is formally stated as follows:
	\begin{equation}
		(\Phi_1): \enspace \underset{q,n,p,P_k}{\mathrm{min}}\quad \varphi(q, n, p, P_k), \label{eq:12}\\
	\end{equation}
	\begin{alignat}{2}
	\hskip 1em \textrm{s.t.} & \quad	\eqref{eq:DP_cons}, \ \eqref{eq:channel_cap_cons} \nonumber\\
    & \quad \epsilon \leq \bar{\epsilon}, \label{eq:eps_cons} \\	
	& \quad	P^{min}_k \le P_k &\le P^{max}_k,  \mbox{ } k \in \mathcal{K}, \label{eq:trans_pow_cons} \\
	& \quad	q \in \mathcal{Q}, \mbox{ } n \in \mathcal{N}, \label{eq:domain_set_q_n} \\
	& \quad	p \in (0,1), \label{eq:domain_set_p}
	\end{alignat}
	where we denote the domain sets of the level quantization parameter $q$ and the noise parameter $n$ as $\mathcal{Q}$ and $\mathcal{N}$, respectively. The domain sets $\mathcal{Q}$, $\mathcal{N}$ and the objective function are defined as follows:
	\begin{equation} \label{domain-set-q}
	 \mathcal{Q} = \{2, \ldots, \lfloor (1 + \min_{k \in \mathcal{K}} P_k^{\max}h_k / \omega_0)^{TW/d} \rfloor - 2\},
  \end{equation}
  \begin{equation}
  \label{domain-set-n}
    \mathcal{N} = \{2, \ldots, \lfloor (1 + \min_{k \in \mathcal{K}} P_k^{\max}h_k / \omega_0)^{TW/d} \rfloor - 2\},  
	\end{equation}
	\begin{align}
	    \varphi(q, n, p, P_k) = \frac{1 + np(1 - p)}{(q - 1)^2}. \label{obj-func-P_1}
	\end{align}
Since the data dimension $d$, the gradient's upper bound $G$, and the number of devices $K$ are often known in advance \cite{Agarwal2018}, we omit $4dG^2 / K$ from the upper bound on $B$ in \eqref{B-upperbound} to obtain the objective function of ($\Phi_1$) as in Eq. \eqref{obj-func-P_1}.
    The constraints of ($\Phi_1$) represent the differential privacy and system implementation requirements.
    In particular, constraint \eqref{eq:DP_cons} guarantees that the framework follows the $(\epsilon, \delta)$-differential privacy.
    Constraints \eqref{eq:channel_cap_cons} and \eqref{eq:trans_pow_cons} capture the channel capacity and transmit power constraints of each device, respectively.
    Constraint \eqref{eq:eps_cons} ensures that the differential privacy budget $\epsilon$, expressed in Eq.  \eqref{new-privacy-budget-estimation}, does not exceed a given upper bound $\bar{\epsilon}$.
    Finally, constraints \eqref{eq:domain_set_q_n} and \eqref{eq:domain_set_p} describe the domain set of the quantization level $q$ and Binomial mechanism parameters $n$ and $p$. The upper bound $\lfloor (1 + \min_{k \in \mathcal{K}} P_k^{\max}h_k / \omega_0)^{TW/d} \rfloor - 2$ of $q$ and $n$ in Eqs. \eqref{domain-set-q} and \eqref{domain-set-n} is derived from the constraint \eqref{eq:channel_cap_cons}.
	
	
Finally, we discuss the relationship between the system parameters, e.g., the maximum transmit power, the bandwidth, the transmission time, and the optimal objective value of ($\Phi_1$).
    Remark \ref{input-parameter-change-affect-on-solution} summarizes the relationship between the optimal objective value of ($\Phi_1$) and these parameters.
	
	\begin{remark} 
	[Dependence of Optimal Solution Value on the System Parameters]
	\label{input-parameter-change-affect-on-solution}
	    If the maximum transmit power $P_k^{\max}$ or the maximum transmission time $T$ or the bandwidth $W$ increases, the optimal objective function of ($\Phi_1$) will not increase.
	\end{remark}
	
\subsection{Problem Transformation}
The proposed problem ($\Phi_1$) is an MINLP problem.
    In this section, we transform ($\Phi_1$) to a new problem ($\Phi_2$) whose optimal solution set can be used to derive the optimal solution set of ($\Phi_1$).
    The advantage of ($\Phi_2$) over ($\Phi_1$) is that it can be effectively solved by approximate solutions with arbitrarily small errors.
    Specifically, applying transformations on the constraints of ($\Phi_1$), we obtain a new MINLP problem denoted as ($\Phi_2$).
	
    \begin{equation}
		(\Phi_2): \enspace \underset{q,n,p,P_k}{\mathrm{min}}\quad \varphi(q, n, p, P_k), \label{obj_func_2}\\
	\end{equation}
	\begin{alignat}{2}
	\textrm{s.t.} & \ \eqref{eq:domain_set_q_n}, \ \eqref{eq:domain_set_p}, \nonumber\\
    &n = \max \Big\{ \Big\lceil \frac{\max\{23 \ln\frac{10d}{\delta}, 2(q + 1)\}} {Kp(1- p)} \Big\rceil, n_1 \Big\}, \label{eq:value_m}
    \end{alignat}
    \begin{alignat}{2}
	&P_k =  \min \Big\{P^{\max}_{k}, \max \Big\{P_k^{\min}, \frac{\omega_0 [ (q + n)^{\frac{d}{TW}} - 1 ]}{h_k}\Big\}\Big\}, \nonumber\\
	& \hspace{100pt} k \in \mathcal{K}, 
 \label{eq:P_k_compute_cons}	
\end{alignat}
    where $\varphi(q, n, p, P_k)$ is defined in \eqref{obj-func-P_1}, and in the constraint \eqref{eq:value_m} $n_1$ is explained in detail as follows.
	
   In particular, to derive \eqref{eq:value_m}, we first observe that the inequality (4) is equivalent to the following inequality
        \begin{equation}
	   n \geq \left\lceil \frac{\max \big\{ 23 \ln\frac{10d}{\delta}, \ 2(q+1) \big\}}{Kp(1-p)} \right\rceil.
        \end{equation}
    Second, the privacy budget $\epsilon$ monotonically decreases with respect to $n$ (more details are provided in Appendix \ref{proof-of-theorem-relationship-P1-P2-models}); therefore, given the quantization level $q$ and Binomial mechanism parameter $p$, there exists an integer $n_1$ such that $\epsilon(n) \leq \bar{\epsilon}$ if and only if $n \geq n_1$.
    $n_1$ is derived by applying the binary search with respect to $n \in \mathcal{N}$.
    Combining these two facts with the monotonic property of the objective function with respect to $n$, we obtain the constraint \eqref{eq:value_m}.
     The constraint \eqref{eq:P_k_compute_cons} follows from the fact that the smaller the transmit power is, the better the power efficiency can be achieved.
	Theorem \ref{relationship-P1-P2-models} formally states the relationship between the problem ($\Phi_1$) and the problem ($\Phi_2$).
    
	\begin{theorem} 
	[Solutions of Problems ($\Phi_1$) and ($\Phi_2$)]
    \label{relationship-P1-P2-models}
	\textit{(i)}: If $(q^*, n^*, p^*, P_k^*)$ is an optimal solution of ($\Phi_2$), $(q^*, n^*, p^*, P^*_k)$ is also an optimal solution of ($\Phi_1$).
	\textit{(ii)}: If ($\Phi_2$) is infeasible, ($\Phi_1$) is also infeasible.
	\textit{(iii)}: The set $\{(q^*, n^*, p^*, P_k)|~P_k^* \leq P_k \leq P_k^{\max} \mbox{ and } (q^*, n^*, p^*, P_k^*) \mbox{ is an optimal solution of } (\Phi_2)\}$ is the optimal solution set of ($\Phi_1$), i.e., we can derive all the optimal solutions of ($\Phi_1$) from optimal solutions of ($\Phi_2$).
	\end{theorem}
	\begin{proof}
	    See Appendix \ref{proof-of-theorem-relationship-P1-P2-models}.
	\end{proof}
	
Theorem \ref{relationship-P1-P2-models} shows that we can obtain the solution of $(\Phi_1)$ by solving $(\Phi_2)$.
    An advantage of $(\Phi_2)$ in comparison with $(\Phi_1)$ is that we only need to consider the variables $p$ and $q$, and easily derive the values of $P_k$ and $n$ based on the Eqs. \eqref{eq:value_m} and \eqref{eq:P_k_compute_cons}.
    In the next section, we propose an approximate algorithm to solve $(\Phi_2)$ that guarantees arbitrary small errors and works effectively in practice.
	
    \begin{algorithm}[t]
	    \caption{Binary Search algorithm to solve the privacy budget constraint with respect to $n$ with fixed $q$ and $p$}
	    \label{alg:m-epsilon-domain}
	    \begin{flushleft}
	        \hspace*{\algorithmicindent} \textbf{Input}: $q$, $p$, and $\bar{\epsilon}$\\
	        \hspace*{\algorithmicindent} \textbf{Output}: $n_1$
	    \end{flushleft}
	    \begin{algorithmic}[1]
	     \STATE $n_l, n_u \leftarrow 2$
          \WHILE {$\epsilon(n_u) > \bar{\epsilon}$}
            \STATE $n_l \leftarrow n_u$
            \STATE $n_u \leftarrow n_u \times 2$
          \ENDWHILE
          \STATE $n_1 \leftarrow \left\lfloor \frac{n_l+ n_u}{2} \right\rfloor $
          \WHILE{$n_u > n_l$}
            \IF{$\epsilon(n_1) > \bar{\epsilon}$}
                \STATE $n_l \leftarrow{n_1}$
            \ELSE
                \STATE $n_u \leftarrow{n_1}$
            \ENDIF
            \STATE $n_1 \leftarrow \left\lfloor \frac{n_l+ n_u}{2} \right\rfloor$
          \ENDWHILE
	   \RETURN $n_1$
	    \end{algorithmic}
    \end{algorithm}
	
\subsection{Approximate Algorithm} \label{alg-solve-p2} 
In this section, we design Algorithm \ref{alg:dis-general-solution} to solve the problem ($\Phi_2$).
	The main idea is to perform a search on the Cartesian product set $\mathcal{Q} \times \mathcal{P}$, where $\mathcal{Q}$ and $\mathcal{P}$ are the finite subsets of the domain sets of the quantization level $q$ and the Binomial mechanism parameter $p$, respectively.
    The quantization level domain set $\mathcal{Q}$ is defined as $\mathcal{Q} = \{2, \ldots, \bar{q}\}$, where the level quantization upper bound $\bar{q}$ is defined by Lemma \ref{thr:narrow-quatization-level-1}.
	The Binomial mechanism parameter domain set $\mathcal{P}$ is defined as $\mathcal{P} = \mathcal{P}_{\lambda} \cup \{1/2\}$.
	Set $\mathcal{P}_\lambda$ contains all elements that are larger than 1/2 and smaller than 1 of the arithmetic progression sequence $i\lambda$ for some $\lambda > 0$ and $i \in \mathbb{N}^+$.
	The restriction $p \geq 1/2$ is explained by Lemma \ref{thr:replace-p-by-1p-if-plt1o2}.
		
	\begin{lemma}
	[Symmetric Property of the Feasible Region of Problem ($\Phi_2$) with Symmetry Point $p = 1/2$] 
    \label{thr:replace-p-by-1p-if-plt1o2}
	    If $(\tilde{q}, \tilde{n}, \tilde{p}, \tilde{P}_k)$ is a feasible solution of ($\Phi_2$) then $(\tilde{q}, \tilde{n}, 1 - \tilde{p}, \tilde{P}_k)$ is also a feasible solution of ($\Phi_2$) with the same objective value.
	\end{lemma}
	\begin{proof}
	    See Appendix \ref{proof-of-replace-p-by-1p-if-plt1o2}.
	\end{proof}
	
	\begin{algorithm}
	    \caption{Approximate algorithm for Problem ($\Phi_2$)}
	    \label{alg:dis-general-solution}
	    \begin{flushleft}
	        \hspace*{\algorithmicindent} \textbf{Input}: Domain sets $\mathcal{Q}$, $\mathcal{N}$, $\lambda$, $\mathcal{P} = \mathcal{P}_{\lambda} \cup \{\frac{1}{2}\}$\\
	        \hspace*{\algorithmicindent} \textbf{Output}: Approximated solution ($\tilde{q}, \tilde{n}, \tilde{p}, \tilde{P}_k$)
	    \end{flushleft}
	    \begin{algorithmic}[1]
	        \STATE $\tilde{\varphi} \leftarrow +\infty$
	        \FOR {$(q, p) \in \mathcal{Q} \times \mathcal{P}$} \label{alg:begin-q-p-for-loop}
	           \STATE Determine $n_1$ by using Algorithm \ref{alg:m-epsilon-domain} \label{alg:line-bin-search-for-m}
	           \STATE Compute $n$ by applying Eq.~\eqref{eq:value_m} \label{alg:compute-m-32}
	           \IF {$n \leq  \left(1 + \min_{k \in \mathcal{K}} \frac{P^{\max}_kh_k}{\omega_0}\right)^{\frac{TW}{d}} - q$} \label{alg:check-final-m-value}
	                \STATE $\varphi \leftarrow \frac{1 + np(1 - p)}{(q - 1)^2}$
	                \IF {$\tilde{\varphi} > \varphi$} \label{begin-update-solution}
	                    \STATE $\tilde{q} \leftarrow q$, $\tilde{n} \leftarrow n$, $\tilde{p} \leftarrow p$, $\tilde{\varphi} \leftarrow \varphi$
	                \ENDIF \label{end-update-solution}
	           \ENDIF
	        \ENDFOR \label{alg:end-q-p-for-loop}
	        \FOR{$k \in \mathcal{K}$} \label{alg:begin-Pk-compute-loop}
	            \STATE $\tilde{P}_k \leftarrow \min \bigg\{ P_k^{\max}, \max \bigg\{ P_k^{\min}, \frac{\omega_0 (\tilde{q} + \tilde{n})^{\frac{d}{TW}} - 1}{h_k}\bigg\}\bigg\}$
	        \ENDFOR \label{alg:end-Pk-compute-loop}
	        \RETURN ($\tilde{q}, \tilde{n}, \tilde{p}, \tilde{P}_k$)
	    \end{algorithmic}
	\end{algorithm}
		
In Algorithm \ref{alg:dis-general-solution}, each iteration of the FOR loop (lines \ref{alg:begin-q-p-for-loop}-\ref{alg:end-q-p-for-loop}) corresponds to a particular pair $(q, p) \in \mathcal{Q} \times \mathcal{P}$.
	First, on the line \ref{alg:line-bin-search-for-m}, we compute $n_1$.
	Second, on line \ref{alg:compute-m-32} we compute the value $n$ by {Eq.}~\eqref{eq:value_m}.
	Third, on line \ref{alg:check-final-m-value} we check \eqref{eq:channel_cap_cons}.
	If it satisfies, we compute the objective value $\varphi$ of ($\Phi_2$) and update the solution (lines \ref{begin-update-solution}-\ref{end-update-solution}).
	Finally, the transmit power is computed (lines \ref{alg:begin-Pk-compute-loop}-\ref{alg:end-Pk-compute-loop}).
	
Lemma \ref{thr:replace-p-by-1p-if-plt1o2} shows that instead of considering $p \in (0, 1)$, we only need to study $p \in [1/2, 1)$.
    Thus, this lemma helps to speed up Algorithm \ref{alg:dis-general-solution}.
    Likewise, in Lemma \ref{thr:narrow-quatization-level-1}, we present an upper bound for the quantization level $q$ that also helps to reduce the running time of Algorithm \ref{alg:dis-general-solution}.
	
	\begin{lemma} [Upper Bound on the Quantization Level $q$ for $p \geq 1/2$] \label{thr:narrow-quatization-level-1}
	For each privacy budget upper bound $\bar{\epsilon}$, there exists an integer $\bar{q}$ such that to satisfy the privacy budget condition $\epsilon < \bar{\epsilon}$, the level quantization $q$ must not exceed $\bar{q}$ for every $p \geq 1/2$. The upper bound $\bar{q}$ can be computed by solving the equation $g(q) = \bar{\epsilon}$, where $g(p)$ defined by Eq.~\eqref{upper-bound-level-quantization-function} is a monotonically increasing function with respect to $q$.
	\end{lemma}
	\begin{proof}
	    See Appendix \ref{proof-of-narrow-quatization-level-1}.
	\end{proof}
	
In the next section, we prove that we can control the relative error of Algorithm \ref{alg:dis-general-solution}.
    Furthermore, in Section \ref{ss:complexity-algs}, we analyze the complexity of this algorithm.
	
\subsection{Relative Error of Algorithm \ref{alg:dis-general-solution}} \label{ss:effectiveness-algs}
Theorem \ref{thr:relative-error-statement} below states that Algorithm \ref{alg:dis-general-solution} can generate a $\rho$-relative error solution ($\tilde{q}, \tilde{n}, \tilde{p}, \tilde{P}_k$), i.e., $\varphi(\tilde{q}, \tilde{n}, \tilde{p}, \tilde{P}_k) / \varphi^* < 1 + \rho$, where $\varphi^*$ is the optimal objective value of ($\Phi_2$).
	In addition, Theorem \ref{thr:relative-error-estimation} gives an approach to compute the value $\lambda$ to guarantee an arbitrary $\rho$ when $\eta = \min\{23 \ln(10d / \delta), 6\} / (K\bar{n}_\mathcal{N}) < 0.25$, where 6 is the smallest value of $2(q + 1)$, which appears in inequality \eqref{eq:DP_cons}, and $\bar{n}_\mathcal{N}$ is the maximum value of Binomial parameter set $\mathcal{N}$.
	Since $\eta \leq 6 / (K \bar{n}_\mathcal{N})$, the condition \mbox{$\eta < 0.25$} holds if the following condition holds:
    \begin{equation} \label{relative-error-condition}
        K \ge 12 \mbox{ or } \bar{n}_\mathcal{N} \ge 12.
    \end{equation}
    In other words, if the condition \eqref{relative-error-condition} holds, the relative error is bounded as shown in Theorem \ref{thr:relative-error-estimation}: $\rho < \mu \lambda$, where $\mu = 2/\big(1 - \sqrt{1 - 4\eta}\big)$.
	Since in practice, the FLoWNs often contain thousands of devices \cite{Agarwal2018}, the condition $K \ge 13$ is likely to occur.
    Even in the case $K < 13$, we can add 13 to $\mathcal{N}$ without losing the generality to guarantee that the condition \eqref{relative-error-condition} always holds.
 	
	\begin{theorem} 
	[Relative Error Property of Solution obtained by Algorithm \ref{alg:dis-general-solution}]
	\label{thr:relative-error-statement}
	    For arbitrary $\rho > 0$, there exists a positive $\bar{\lambda}$ such that for every $\lambda < \bar{\lambda}$, Algorithm \ref{alg:dis-general-solution} with $\mathcal{P}_\lambda = \{i\lambda | i \in \mathbb{N}^+, 1/2 < i\lambda < 1\}$ returns a feasible solution $(\tilde{q}, \tilde{n}, \tilde{p}, \tilde{P}_k)$ satisfying $\rho$-relative error for ($\Phi_2$), and this solution is also a $\rho$-relative error solution of ($\Phi_1$), i.e.:
     $$\frac{\varphi(\tilde{q}, \tilde{n}, \tilde{p}, \tilde{P}_k)}{\varphi^*} < 1 + \rho.$$
	\end{theorem}
	\begin{proof}
	    See Appendix \ref{proof-theorem-relative-error-statement}.
	\end{proof}
	
	\begin{theorem} 
	[Relative Error's Upper Bound on Solution Returned by Algorithm \ref{alg:dis-general-solution}]
    \label{thr:relative-error-estimation}
	    For $\lambda < \bar{\lambda}$, where $\bar{\lambda}$ is defined in Theorem \ref{thr:relative-error-statement}, and $\eta < 0.25$, the relative error $\rho$ of the solution returned by Algorithm \ref{alg:dis-general-solution} satisfying: $\rho < \mu\lambda$, where $\mu = 2 / (1 - \sqrt{1 - 4\eta})$.
	\end{theorem}
	\begin{proof}
	See Appendix \ref{proof-thr:relative-error-estimation}.
	\end{proof}
	
Applying Theorem \ref{thr:relative-error-estimation}, we can estimate $\lambda$ to guarantee that Algorithm \ref{alg:dis-general-solution} returns a feasible solution satisfying a given relative error $\rho > 0$.
    In general, we select value $\lambda < \rho / \mu$.
    Specifically, for the case $\eta << 1$, since $\mu = 2 / (1 - \sqrt{1 - 4\eta}) = 2 (1 + \sqrt{1 - 4\eta}) / (4\eta) \approx 1 / \eta$, we can select $\lambda$ such that $\lambda < \rho \eta$.
	
\subsection{Complexity of Algorithm \ref{alg:dis-general-solution}} \label{ss:complexity-algs}
In Algorithm \ref{alg:m-epsilon-domain}, we implement a binary search with respect to the Binomial mechanism parameter $n$.
    Therefore, the complexity of Algorithm \ref{alg:m-epsilon-domain} is $\mathcal{O}(\log_2|\mathcal{N}|)$, where $\mathcal{N}$ is the domain set of $n$.
    In practice, we assign $\mathcal{N}$ to the set of all integers between 2 and $2^b$, where $b$ is the maximum number of bits of each gradient element, e.g., 8 or 16 \cite{Agarwal2018}.
    In such context, the complexity of Algorithm \ref{alg:m-epsilon-domain} is $\mathcal{O}(b)$.
	
	Considering Algorithm \ref{alg:dis-general-solution}, the FOR loop (lines \ref{alg:begin-q-p-for-loop}-\ref{alg:end-q-p-for-loop}) repeats for $|\mathcal{Q}| |\mathcal{P}|$ times.
	Inside this loop, the most significant computation is the binary search (line \ref{alg:line-bin-search-for-m}) whose complexity is $\mathcal{O}(\log_2|\mathcal{N}|)$.
	Therefore, the complexity of Algorithm \ref{alg:dis-general-solution} is $\mathcal{O}(|\mathcal{Q}| |\mathcal{P}| \log_2|\mathcal{N}|)$.
	Since $|\mathcal{P}| = \mathcal{P}_\lambda \cup \{1/2\} \leq 1/(2\lambda) + 1$, the complexity of Algorithm \ref{alg:dis-general-solution} is $\mathcal{O}(|\mathcal{Q}| \log_2|\mathcal{N}| / \lambda)$, which is pseudo-polynomial.
    Recall that to satisfy the $\rho$-relative error the following condition holds: $\lambda < \rho / \mu$.
    Since $\mu$ is a constant, to achieve $\rho$-relative error, the complexity of Algorithm \ref{alg:dis-general-solution} is $\mathcal{O}(|\mathcal{Q}|\log_2|\mathcal{N}| / \rho)$.

\section{Experiment results}\label{s:numerical-results}
\subsection{Parameters Settings}
To perform the experiments, we consider a network with 1 million mobile devices, i.e., $M = 10^6$, and in general, set the number of selected mobile devices to update the global gradient $K = 1000$.
    In a special case that studies the impact of the number of selected devices $K$ on the convergence rate, we vary the value of $K$ between $10$, $10^2$, $10^3$, and $10^4$.
	We set the square of channel gains $h_k^2$ following the exponential distribution with the mean $g_0(D_0/D_k)^4$, where $g_0 = -40$ dB, the reference distance $D_0 = 1$ m, and distance $D_k$ between the mobile edge server and device $k$ is randomly sampled from $[D_{\min}, D_{\max}]$ with $D_{\min} = 2$ m and $D_{\max} = 200$ \cite{Tran2019}.
	For each device, the bandwidth is set to 900 MHz \cite{etsi138_104}.
	The transmit power is restricted as $P_k \in \{1, \ldots, 20\}$ dBm, for $k \in \{1,\ldots,M\}$, similar to \cite{Tran2019}.
	We implement a three-layer neural network with 785 nodes in the input layer, 60 hidden nodes, and 10 nodes in the output layer using the ADAM training algorithm. We also use the ReLU activate function and use the infinite MNIST dataset as input, similar to \cite{Agarwal2018}.
    The simulated framework is built with Python and NumPy.
    For the differential privacy security, we set $\delta = 10^{-10}$ \cite{Agarwal2018}.
    In addition, we restrict the number of allowed transmit bits per parameter to 16 bits, similar to \cite{Agarwal2018}.

\subsection{Privacy evaluation} 
We first aim to study the efficiency of the proposed privacy budget estimation \eqref{new-privacy-budget-estimation} in comparison with the privacy budget estimation \eqref{eq:expression-epsilon} presented in \cite{Agarwal2018}.
    We compute the solutions $(q, n, p, P_k)$ with respect to the privacy budget upper bound $\bar{\epsilon}$ which is varied from 1 to 10 using the proposed privacy budget estimation in Eq. \eqref{new-privacy-budget-estimation}.
    We then compute the privacy budget estimation over these solutions applying Eq. \eqref{eq:expression-epsilon}.
    The results in Fig. \ref{fig:efficiency_epsilon_estimation} show that our proposed privacy budget estimation gives a tighter estimation than the proposed one in \cite{Agarwal2018}, i.e., the right-hand side of Eq. \eqref{new-privacy-budget-estimation} is smaller than the right-hand side of Eq. \eqref{eq:expression-epsilon} with respect to the same system and quantization/noise parameters.
    \begin{figure}[t]		\centerline{\includegraphics[width=0.5\textwidth]{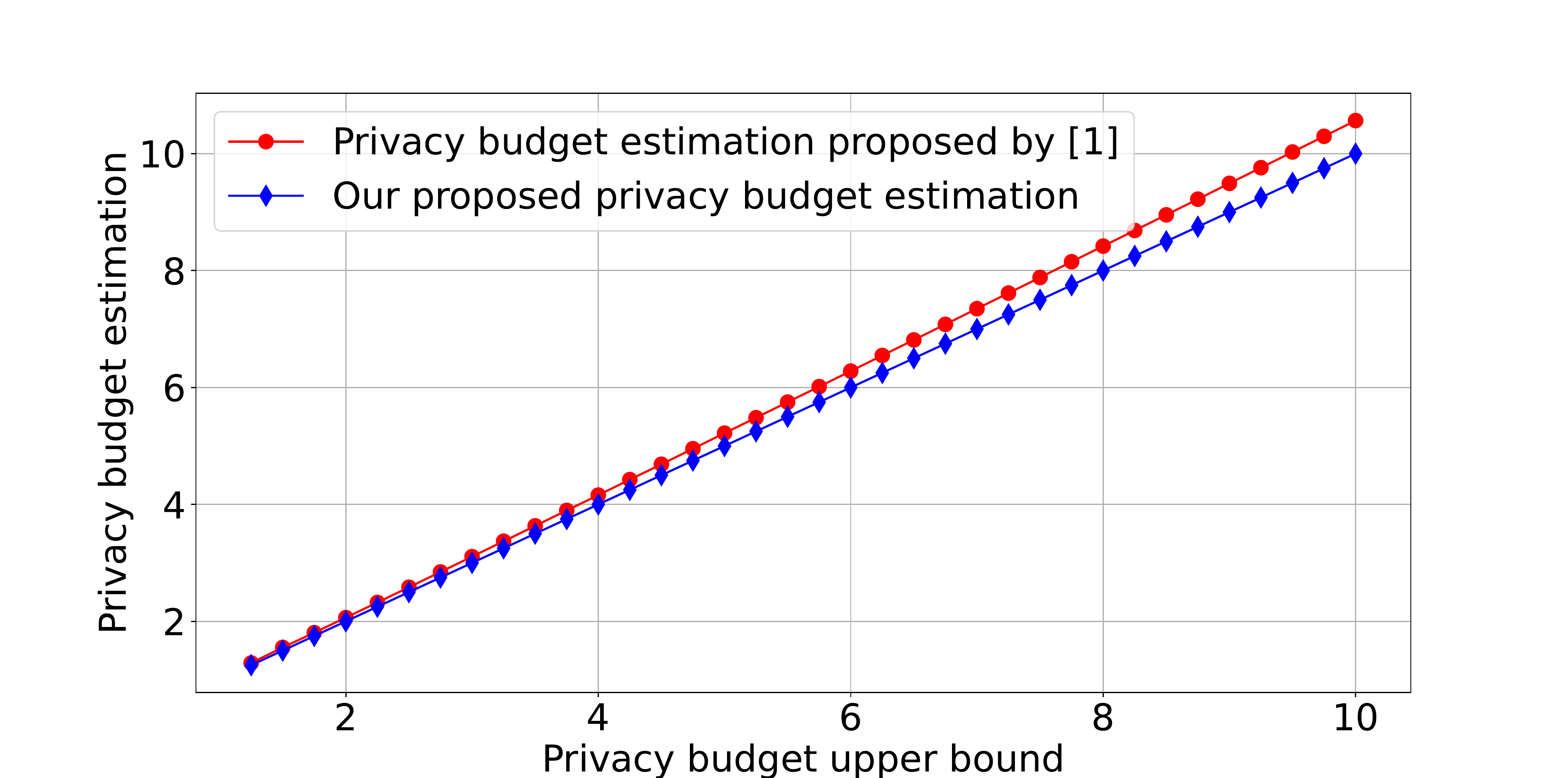}}
		\caption{The estimated privacy budget (blue line) \eqref{eq:expression-epsilon} computed on the solutions of ($\Phi_2$) solved applying the novel estimated privacy budget \eqref{new-privacy-budget-estimation} (red line).}
  \label{fig:efficiency_epsilon_estimation}
	\end{figure}

 \begin{figure}[t]
    \centerline{\includegraphics[width=0.5\textwidth]{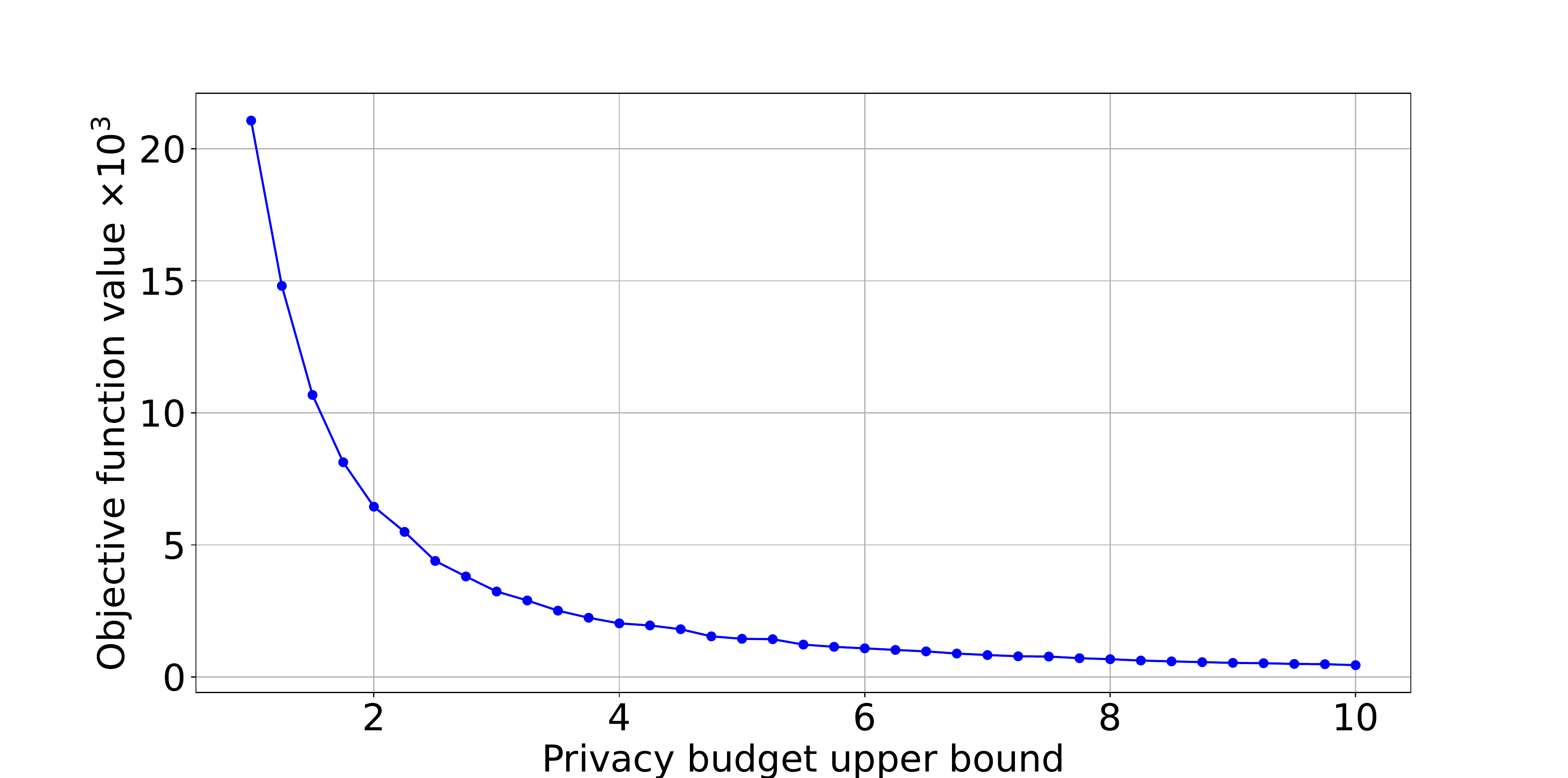}}
		\caption{The objective function value of ($\Phi_2$) returned by Algorithm \ref{alg:dis-general-solution} when varying privacy budget upper bound $\bar{\epsilon}$.}
		\label{fig:obj-func-over-privacy-budget}
	\end{figure}
    
We then investigate the impact of the maximum privacy budget on the learning time of SLQBM-FL.
    Recall that the objective of ($\Phi_2$) is to maximize the SLQBM-FL learning rate.
    Thus, the lower value of ($\Phi_2$)'s objective function is, the lower the learning time of SLQBM-FL is.
    We study the scenarios corresponding to the privacy budget $\bar{\epsilon}$ varying from 1 to 10 \cite{Agarwal2018}.
    Fig. \ref{fig:obj-func-over-privacy-budget} shows the objective function value $\varphi(q, n, p, P_k)$ of ($\Phi_2$) of the solution returned by Algorithm \ref{alg:dis-general-solution} when the upper bound on differential privacy budget $\bar{\epsilon}$ varies from 1 to 10.
	It is clear that as $\bar{\epsilon}$ increases, i.e., the privacy requirement gets less restricted, the objective function value gets decreased, meaning that the convergence rate increases.
	This is stemmed from the fact that the higher the value of $\bar{\epsilon}$ is, the lower the amount of noise added to gradients by the Binomial mechanism is.
	Consequently, the learning time (indicated via our objective function) reduces due to less noise as $\bar{\epsilon}$ increases from 1 to 10.	

\subsection{Convergence rate analysis}
Now, we investigate the dependence of the convergence rate on the device number $K$.
    Theoretically, the upper bound $B$ of the quadratic bias introduced by mechanism $\mathcal{M}$, which defines the objective value of both ($\Phi_1$) and ($\Phi_2$), is inversely proportional to $K$.
    Therefore, the larger the value of $K$, the higher the convergence rate is.
    Fig. \ref{fig:acc-opt-sols-varying-K} shows the accuracy curves of the solutions returned by Algorithm \ref{alg:dis-general-solution} corresponding to different values of the device number $K =$ $10$, $10^2$, $10^3$, and $10^4$.
    It is clear that as $K$ increases the convergence rate increases since the aggregated noise reduces as the number of aggregated devices increases.
    
    \begin{figure}[t]		\centerline{\includegraphics[width=0.5\textwidth]{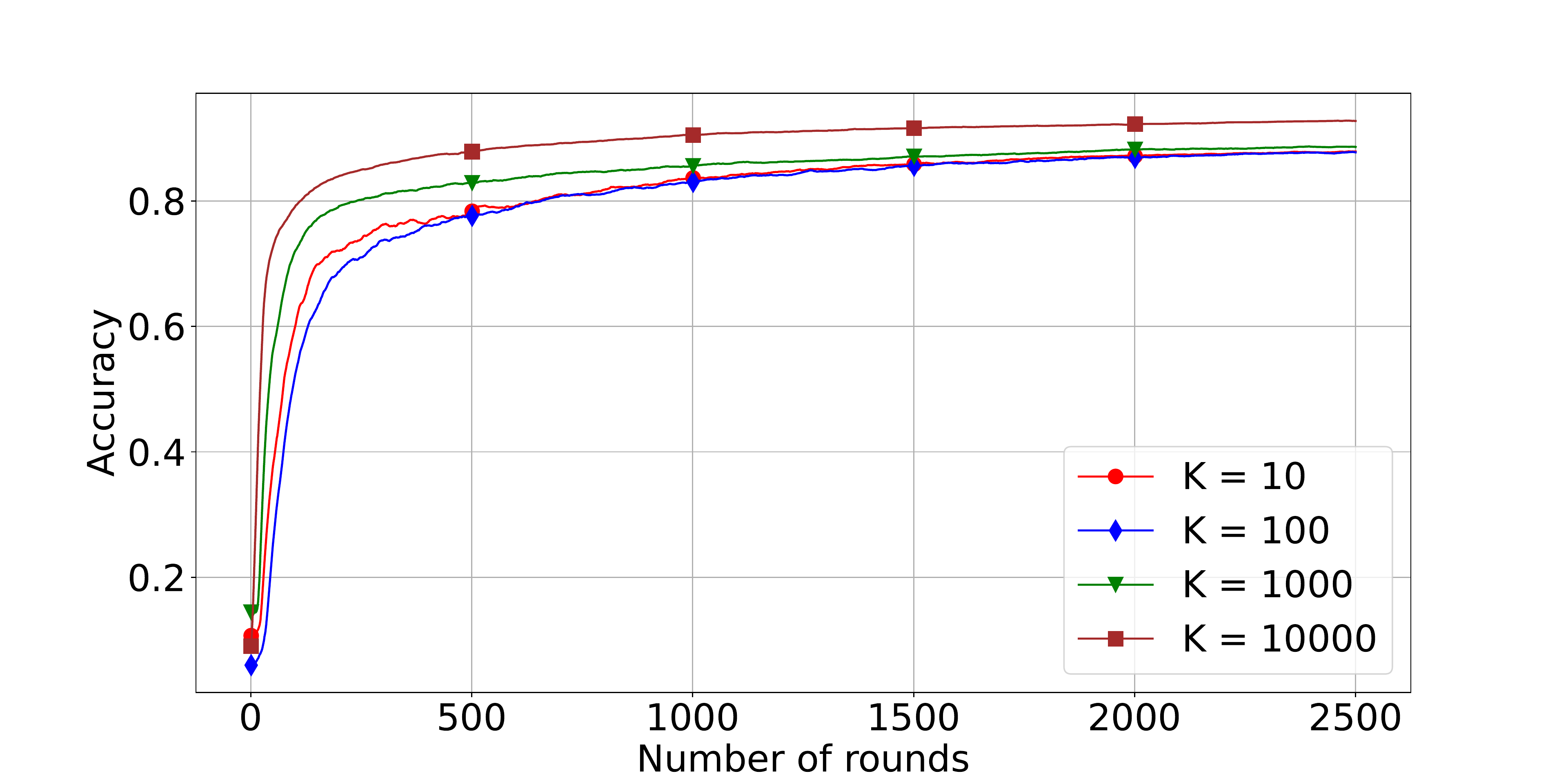}}
		\caption{The accuracy curves of the solutions returned by Algorithm \ref{alg:dis-general-solution} when varying the number of selected devices $K$.}
		\label{fig:acc-opt-sols-varying-K}
	\end{figure}
 	
	\begin{figure}[!]		\centerline{\includegraphics[width=0.5\textwidth]{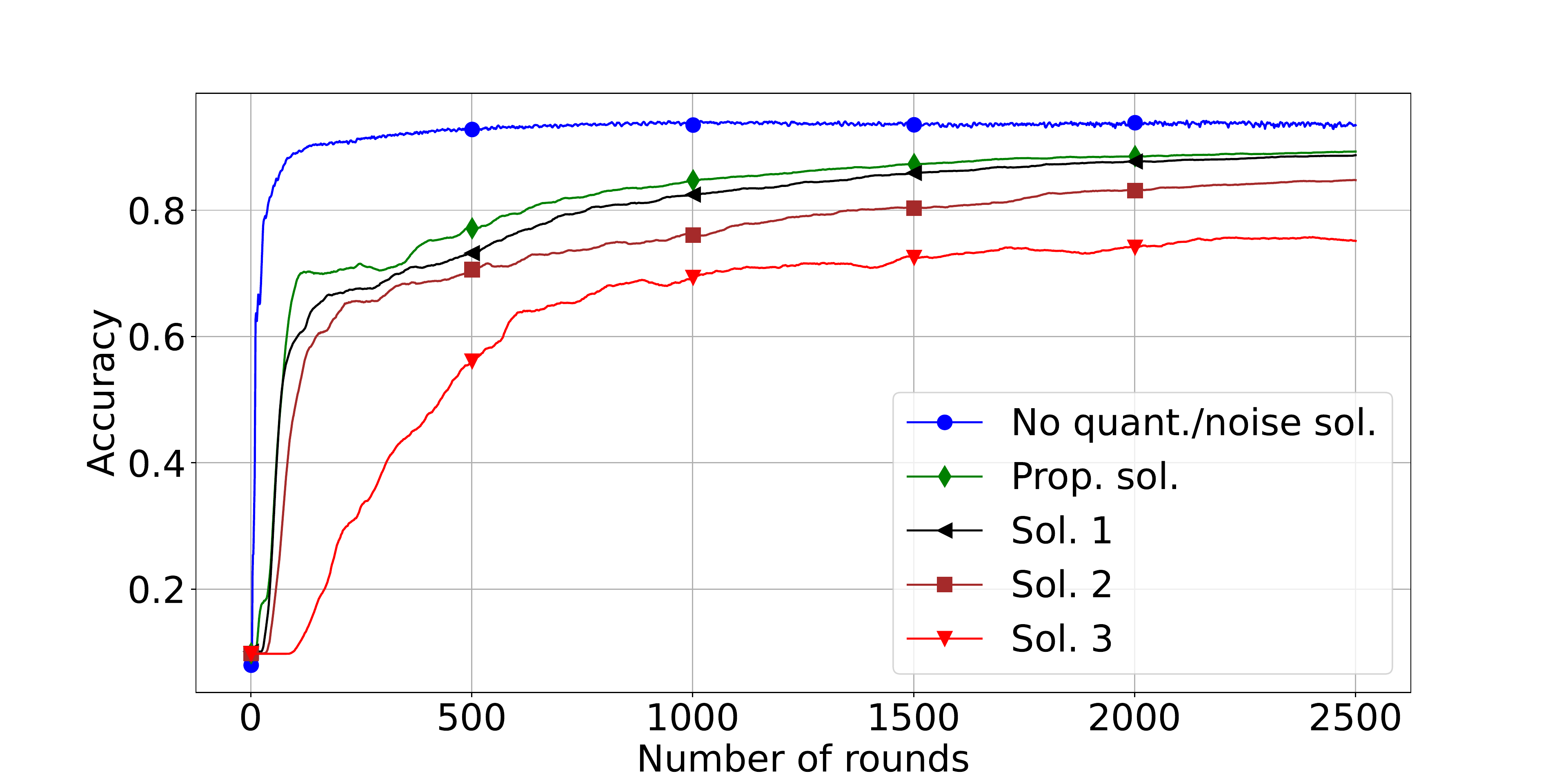}}
		\caption{The accuracy curves of the proposed solution and feasible solutions with the baseline of no quantization/noise}.
		\label{fig:acc-feas-sols-vs-opt-sol}
	\end{figure}

We then investigate the convergence of SLQBM-FL with the parameters of quantization and Binomial mechanisms obtained by our proposed Algorithm \ref{alg:dis-general-solution}, namely \textit{Prop. sol.}.
    In this experiment, we select four baseline approaches, i.e., the conventional FL that operates without quantization and differential privacy mechanisms, and three feasible solutions to the problem ($\Phi_2$).
    In particular, the feasible solutions returned by Algorithm \ref{alg:dis-general-solution} are clustered into 3 groups based on their objective function value.
	We then select the solution with the smallest objective value of each group.
    As shown in Fig. \ref{fig:acc-feas-sols-vs-opt-sol}, even though the quantization and Binomial mechanism introduce noise to data, our proposed solution still achieves an accuracy that is close to that of the conventional FL after $2500$ global update rounds.
    We also observe that the accuracy curve of Sol. 1 gets close to the accuracy curve of our proposed solution.
    The reason is that the objective value of Sol. 1 is the second smallest and is close to the objective value of our proposed solution.
    Thus, it demonstrates the effectiveness of our proposed algorithm in optimizing the system parameters, i.e., transmit power and the quantization/noise parameters for the SLQBM-FL.
    
	

\subsection{System efficiency}
Next, we investigate the dependence of the objective function of ($\Phi_2$) on the system parameters, including the maximum transmit power $P_k^{\max}$, the bandwidth $W$, and the transmission time $T$, as shown in Figs. \ref{fig:obj-value-vs-transmit-power}, \ref{fig:obj-value-vs-bandwidth}, and \ref{fig:obj-value-vs-transmit-time}, respectively.
    Generally, as these system parameters increase, the domain sets of these parameters get expanded. 
    Thus, the objective function value of ($\Phi_2$) decreases or at least does not increase.
    In other words, the convergence rate gets improved as increasing these communication resources.

\begin{figure}[t]		\centerline{\includegraphics[width=0.5\textwidth]{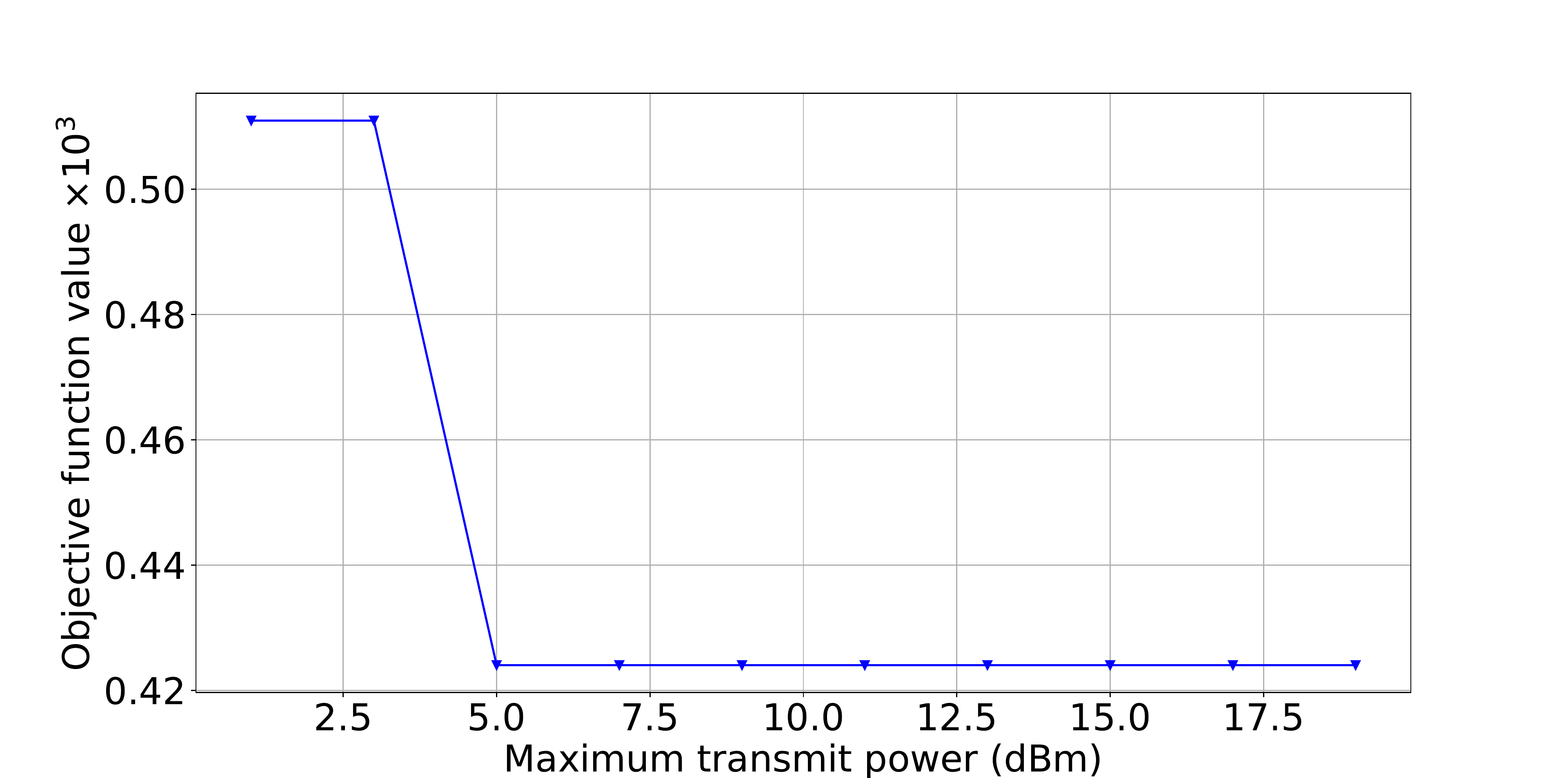}}
		\caption{The objective function of ($\Phi_2$) returned by Algorithm \ref{alg:dis-general-solution} when varying the maximum transmit power.}
		\label{fig:obj-value-vs-transmit-power}
	\end{figure}

\begin{figure}[!]		\centerline{\includegraphics[width=0.5\textwidth]{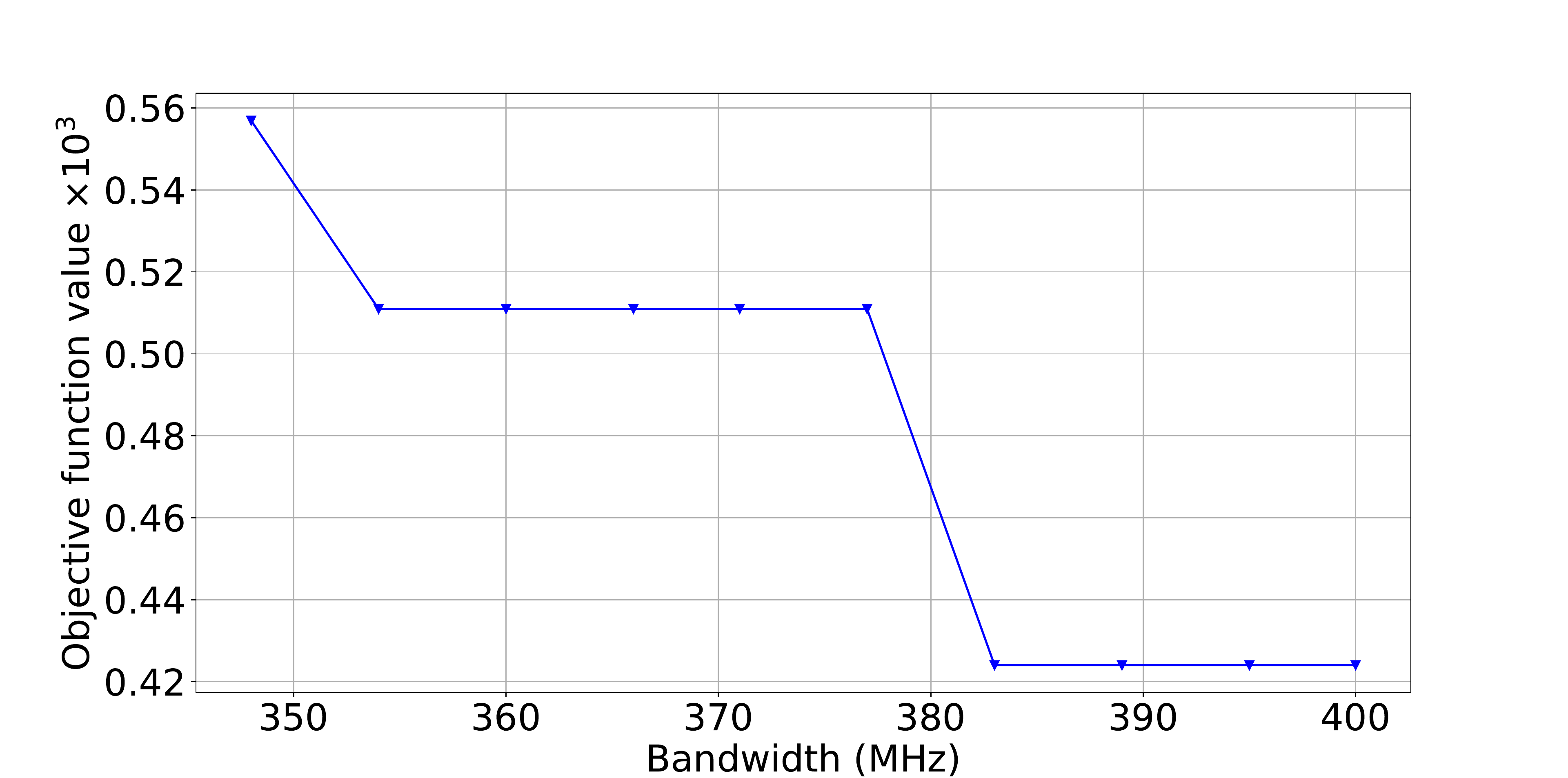}}
		\caption{The objective function of ($\Phi_2$) returned by Algorithm \ref{alg:dis-general-solution} when varying the bandwidth.}
		\label{fig:obj-value-vs-bandwidth}
	\end{figure}

\begin{figure}[!]		\centerline{\includegraphics[width=0.5\textwidth]{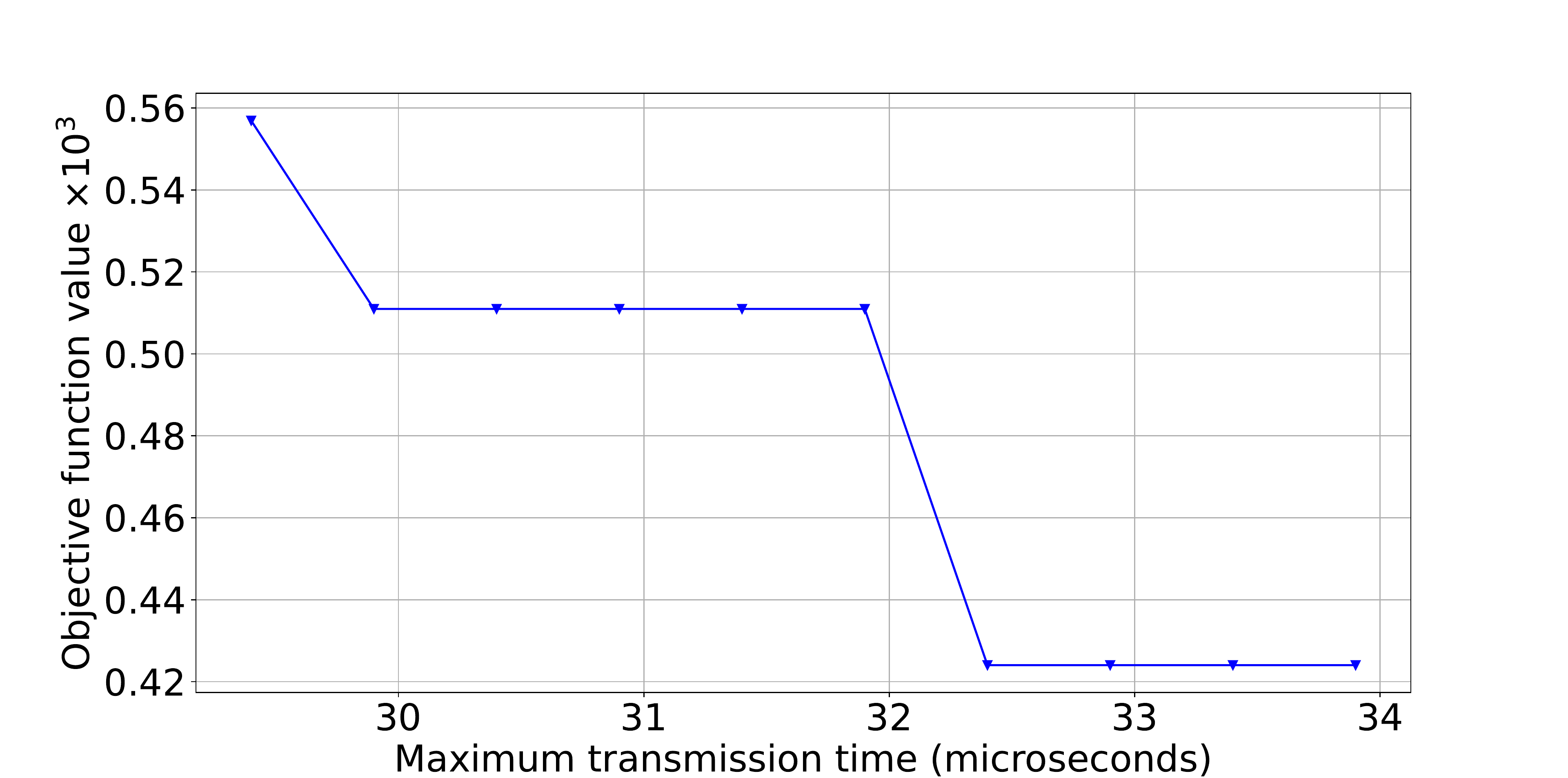}}
		\caption{The objective function of ($\Phi_2$) returned by Algorithm \ref{alg:dis-general-solution} when varying the maximum transmission time.}
		\label{fig:obj-value-vs-transmit-time}
	\end{figure}

To study the communication cost-effectiveness of the proposed algorithm, we compute the communication costs in Giga bits (Gbs) as the product of 4 integers including the number of training rounds, the number of selected mobile devices, i.e., $K = 1000$, the number of dimensions $d$, and the number of bits per quantized gradient element.
    In particular, the number of training rounds counts the number of global update rounds until the accuracy achieves 88\%.
    The reason behind the value 88\% is that this value is the accuracy threshold of the most feasible solutions in empirical.
    The communication costs of the proposed Algorithm \ref{alg:dis-general-solution} and four baseline approaches are presented in a bar chart in Fig. \ref{fig:com_costs}.
    The result of \textit{Sol. 3} is not presented here as its accuracy cannot achieve the threshold of 88\%. 
    It is clear that the proposed algorithm solution's communication cost is close to conventional FL and significantly lower than those of the other feasible solutions.
    In addition, the communication cost of the proposed algorithm is also lower than that of the FL with no quantization but with added noise, showing the effectiveness of the quantization mechanism.

\begin{figure}[t]		\centerline{\includegraphics[width=0.5\textwidth]{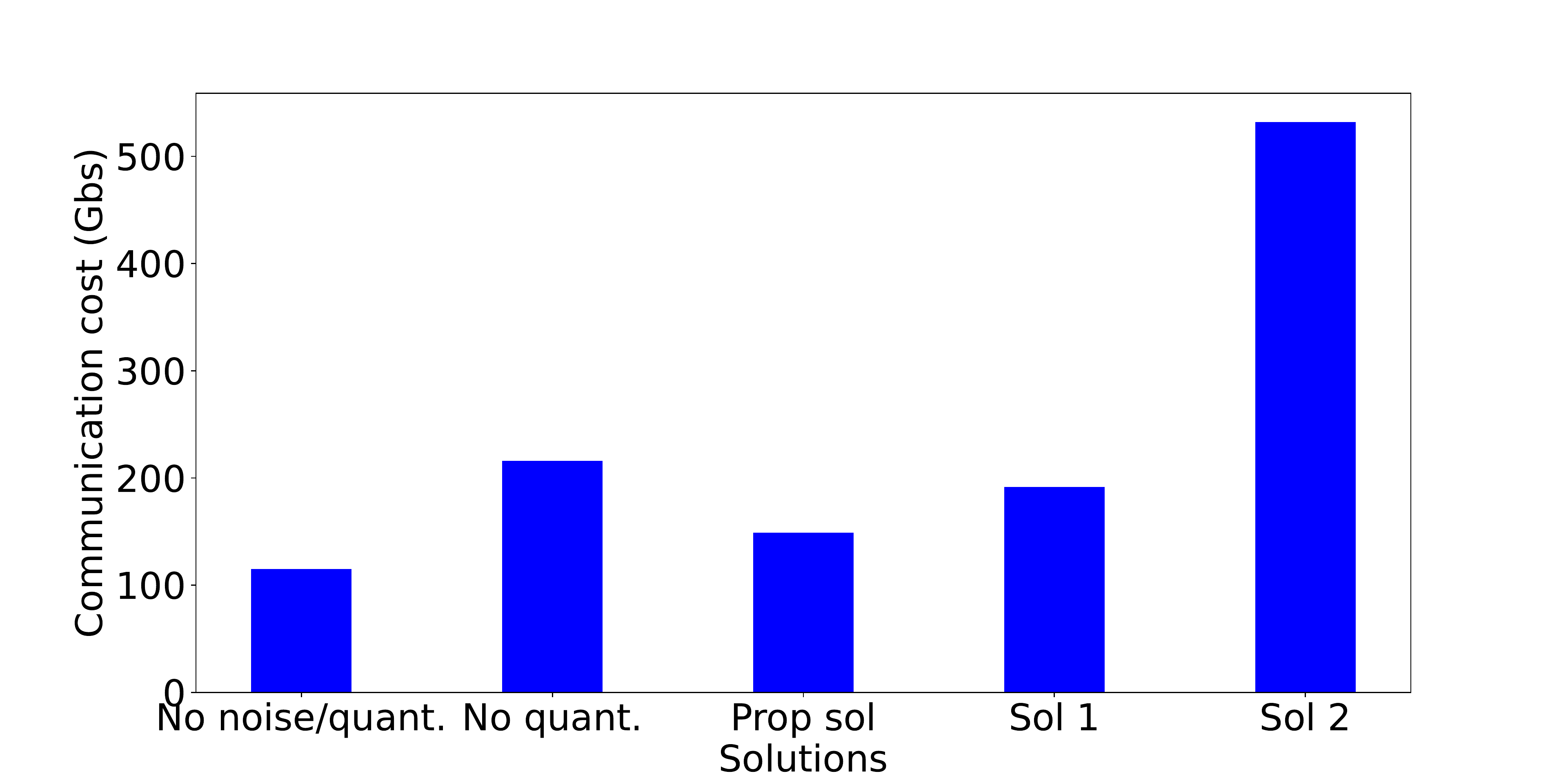}}
		\caption{The communication costs of the proposed solution and feasible solutions with the baseline of no quantization/noise and no quantization but with added noise.}
		\label{fig:com_costs}
	\end{figure}

Finally, we investigate the level quantization upper bound $\bar{q}$ derived from Lemma \ref{thr:narrow-quatization-level-1}.
    Figure \ref{fig:level-quantization-vary-Pk} shows this upper bound when varying the maximum transmit power.
    First, it shows that as the transmit power increases, the level quantization upper bound also increases.
    Second, it shows the efficiency of Lemma \ref{thr:narrow-quatization-level-1} to reduce the search range of level quantization $q$.
    In particular, when we set the number of bits per gradient element not to exceed 16 \cite{Agarwal2018}, instead of searching 
    the range from $2$ to $2^{16}$, we only need to consider the range from $2$ to $\bar{q}$, which is less than $2^{16}$ by a factor of 100, as shown in Fig. \ref{fig:level-quantization-vary-Pk}.

\begin{figure}[!]		\centerline{\includegraphics[width=0.5\textwidth]{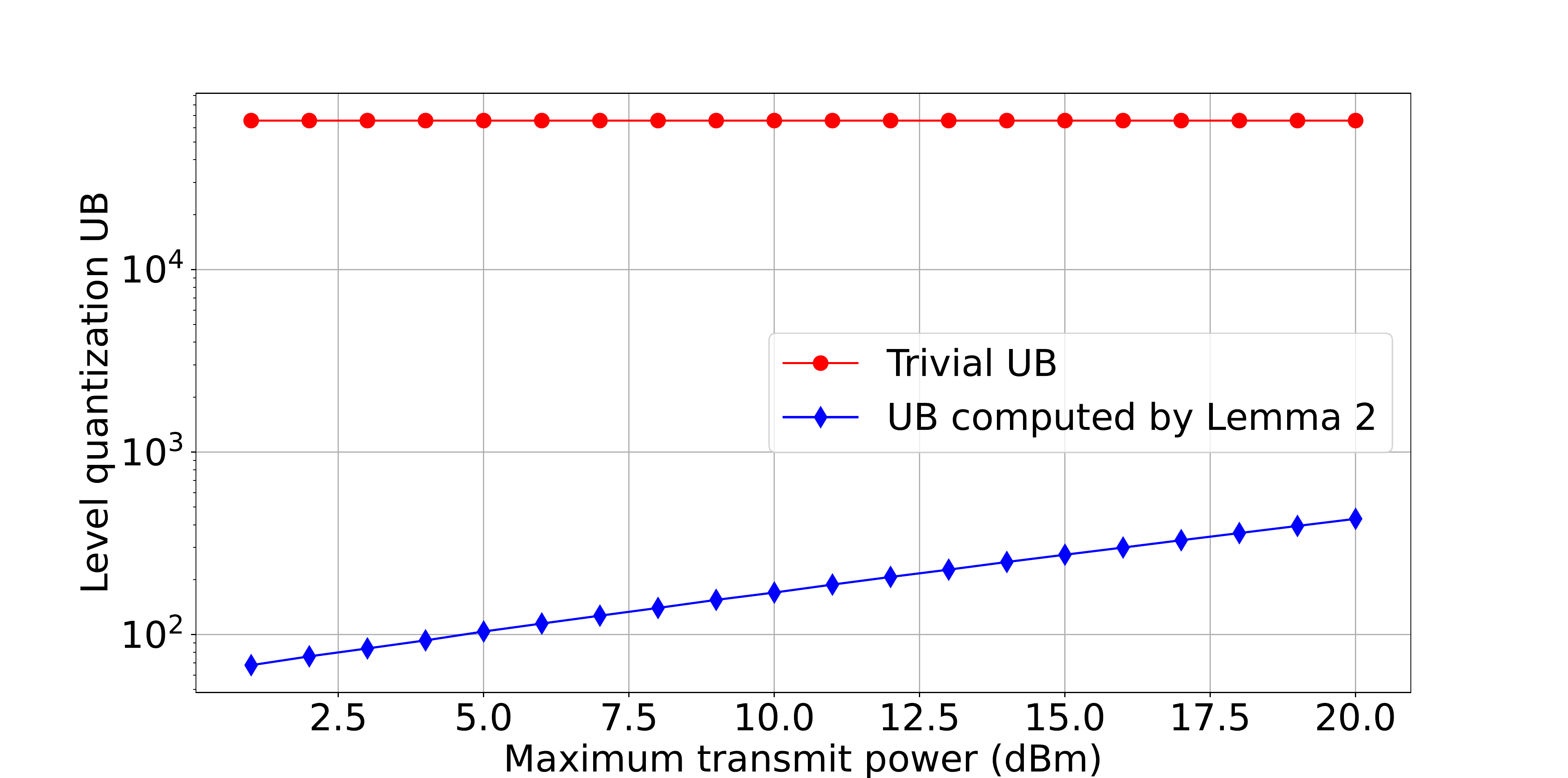}}
		\caption{The log plot of the level quantization upper bound v.s. the maximum transmit power.}
		\label{fig:level-quantization-vary-Pk}
	\end{figure}

\section{Conclusion} \label{s:conclusions}
In this paper, we first derived a  tighter than the state-of-the-art differential privacy budget estimation for Federated learning over mobile edge networks with quantized gradients and added noise (to protect mobile users' privacy).
    By analyzing the relationship between the convergence rate and the system parameters, i.e., the transmit power, the bandwidth, the transmission time, and the quantization/noise parameters, we provided a theoretical bound on the convergence rate.
    This bound was decomposed into two components, including the variance of the global gradient and the quadratic bias introduced by the quantization/noise mechanism.
    We then analyzed the theoretical and practical reasons to improve the convergence rate by optimizing the upper bound of the quadratic bias.
    We then jointly optimized the level of quantization, the Binomial mechanism's parameters, and the transmit power.
    The resulting problem is an MINLP.
    To tackle it, we transformed the problem into the new one whose optimal solutions are proved to be optimal solutions of the original one. We then designed an effective algorithm that can approximately solve the new problem with an arbitrary relative error.
    Extensive simulations showed with mostly the same wireless resources the proposed approach achieves an accuracy that is close to that of the conventional FL without quantization and without added noise while achieving the required DP protection. This suggested the faster convergence rate of the proposed wireless FL framework while optimally preserving users' privacy.

\appendices

\section{Proof of Theorem \ref{therm-new-estimation-privacy-budget}}
\label{new-estimation-privacy-budget-thrm-proof}
To prove Theorem \ref{therm-new-estimation-privacy-budget}, we need the following lemmas.
\begin{lemma}\label{inequality-lemma}
    For any $z \geq -1/3$ the following inequality holds
    \[|\ln{(1+z)} - z| \leq \alpha z^2,\]
    where $\alpha = -3 - 9\ln{\left(2/3\right)}$.
\end{lemma}

\begin{lemma}\label{basic-lemma}
    Considering the event $A$  that $||v_i -np||_{\infty} \leq \beta$ for $d$ Binomial variables $v_i \sim \mathcal{B}(n, p)$, $i \in \{1, \ldots, d\}$, and some $\beta \leq n\min\{p, 1-p\}/3$. For any $\delta > 0$ and $d$ arbitrary real numbers $t_1, t_2, \ldots, t_d$, $X$ denotes the following event:
    \begin{align}
        &\sum_{i = 1}^{d} t_i \left(\ln{\frac{(v_i+1)(1-p)}{(n-v_i+1)p}} - \frac{v_i+1}{np} + \frac{n-v_i+1}{n(1-p)}\right) \nonumber\\
        &\leq \frac{||\mathbf{t}||_1\alpha(np(1-p)+1)(p^2+(1-p)^2)}{\Pr(A)n^2p^2(1-p)^2} + \nonumber\\
        &\frac{||\mathbf{t}||_2\sqrt{S_1}}{\sqrt{\Pr(A)}}\sqrt{2\ln{\left(\frac{1}{\delta}\right)}} + \nonumber\\ &\frac{2||\mathbf{t}||_{\infty}\alpha(\beta+1)^2(p^2+(1-p)^2)}{3n^2p^2(1-p)^2}
        \ln{\left(\frac{1}{\delta}\right)}\label{bernstein-cons-lemma}
    \end{align}
where $S_1$ is defined by Eq. \eqref{s-1-formula} and $\mathbf{t}=(t_1, t_2, \ldots, t_d)$.
We have $\Pr(X|A) \geq 1 - \delta$.
\end{lemma}
    We prove Lemma \ref{basic-lemma} in a similar way to the proof of Lemma 6 in \cite{Agarwal2018}, except for the following points.
    \begin{itemize}
        \item Instead of the inequality $|\ln{(1+z)} - z| \leq 1.95z^2/3$ as in \cite{Agarwal2018}, we use the stronger inequality $\ln{(1+z)} - z| \leq \alpha z^2$, where $\alpha = -3 - 9\ln{\left(2/3\right)} < 1.95/3$.
        \item We replace the term $2/(\omega(n-\omega))$ by $2/((\omega+1)(n-\omega))$ in the following inequality \cite{Agarwal2018}
        \begin{align*}            \bigg(\ln{\bigg(1+\frac{1}{\omega+1}\bigg)} + \ln{\bigg(1+\frac{1}{n-\omega}\bigg)} \\
        - \frac{1}{np(1-p)}\bigg)^2
        \leq \frac{1}{(\omega + 1)^2} + \frac{1}{(n-\omega)^2} + \frac{2}{\omega(n-\omega)} \\
        + \frac{1}{n^2\omega^2(1-p)^2}
        -\frac{2}{np(1-p)} \bigg(\frac{1}{\omega+1} - \frac{1}{2(\omega+1)^2} \\
        + \frac{1}{n-\omega} - \frac{1}{2(n-\omega)^2}\bigg),
        \end{align*}
        where $\omega$ is the sum of $n-1$ arbitrary independent Bernoulli variables.
        \item Instead of the inequality \cite{Agarwal2018}
        \[\mathbb{E}\left[\frac{1}{\omega(n-\omega)}\right] \leq \mathbb{E}\left[\frac{1}{\omega}\right] \mathbb{E}\left[\frac{1}{n-\omega}\right],\]
        we apply the following equality
        \[\mathbb{E}\left[\frac{1}{\omega(n-\omega)}\right] = \frac{1}{n} \left( \mathbb{E}\left[\frac{1}{\omega}\right] +  \mathbb{E}\left[\frac{1}{n-\omega}\right] \right).\]
        \item We replace the inequality \cite{Agarwal2018}
        \[\mathbb{E}\left[\frac{\omega!}{(\omega+i)!}\right] \leq \frac{1}{(np)^i}\]
        by a stronger inequality as follows
        \begin{equation}\mathbb{E}\left[\frac{\omega!}{(\omega+i)!}\right] \leq \frac{1}{n(n+1)\ldots(n+i-1)p^i}.\label{ieqlt:exp-bin-1}\end{equation}
        \item Applying the following equality
        \[\frac{1}{(\omega+1)^2} = \sum_{j=2}^{\infty} \frac{(j-2)!}{(\omega+1)(\omega+2)\ldots(\omega+j)}\]
        we obtain the following inequalities,
        \begin{align}
        &\mathbb{E}\left[\frac{1}{(\omega+1)^2}\right] \leq \frac{1}{n(n+1)p^2} + \nonumber \\
        &\quad \quad \quad \quad \quad \quad \quad \frac{2}{n(n+1)(n+2)p^3},\label{ieqlt:exp-bin-2}\\
        &\mathbb{E}\left[\frac{1}{(n-\omega)^2}\right] \leq \frac{1}{n(n+1)(1-p)^2} + \nonumber\\ 
        &\quad \quad \quad \quad \quad \quad \quad \frac{2}{n(n+1)(n+2)(1-p)^3},\label{ieqlt:exp-bin-3}\\
        &\mathbb{E}\left[\frac{1}{(\omega+1)^2(\omega+2)}\right] \leq \frac{2}{n(n+1)(n+2)p^3},\label{ieqlt:exp-bin-4}\\
        &\mathbb{E}\left[\frac{1}{(n-\omega)^2(n-\omega+1)}\right] \leq \frac{2}{n(n+1)(n+2)} \nonumber\\
        &\quad \quad \quad \quad \quad \quad \quad \quad \quad \quad 
        \quad \quad \quad \frac{1}{(1-p)^3}.\label{ieqlt:exp-bin-5}
        \end{align}
    \end{itemize}
    
We now sketch the proof for Lemma \ref{basic-lemma}. Similar to the proof of Lemma 6 in \cite{Agarwal2018}, we have Bernstein's inequality which states the following result:
\begin{align*}
    \Pr\left(\sum X_i \geq \sum \mathbb{E}[X_i] + \sqrt{2\sum\sigma_i^2 \ln{
    \frac{1}{\delta}}} + \frac{2}{3}M\ln{\frac{1}{\delta}}\right) \\
    \leq \delta,
\end{align*}
for any $\delta > 0$, where $X_1, \ldots, X_d$ are independent random variables satisfying  $|X_i| < M$ and $\sigma_i^2 = \mathbb{E}[X_i^2]$.
    The main idea is to apply Bernstein's inequality with
    \[X_i = t_i \left(\ln{\frac{(v_i+1)(1-p)}{(n-v_i+1)p}} - \frac{v_i+1}{np} + \frac{n-v_i+1}{n(1-p)}\right).\]

First, by applying Lemma \ref{inequality-lemma}, we bound the value of $X_i$ as follows
\begin{align*}
    |X_i| \leq \alpha t_i \frac{(\beta + 1)^2 \left( p^2 + (1-p)^2\right)}{n^2p^2(1-p)^2},
\end{align*}
then we set $M$ as following,
\begin{align*}
    M = \alpha ||\mathbf{t}||_{\infty} \frac{(\beta + 1)^2 \left( p^2 + (1-p)^2\right)}{n^2p^2(1-p)^2}.
\end{align*}
Second, the expectation of $X_i$ conditioned on $A$ is as follows
\begin{align}
    \mathbb{E}[X_i|A] \leq \frac{\alpha |t_i|}{\Pr(A)} \frac{\left(np(1-p) + 1\right)\left(p^2 + (1-p)^2\right)}{n^2p^2(1-p)^2}. \nonumber
\end{align}

Next, to bound the expectation value of $X_i^2$, which denoted as $\sigma_i^2$, we apply the Efron-Stein inequality:
\begin{align*}
    \mbox{Var}(f) \leq \frac{n}{2} \mathbb{E}[\left(f(Z_1, Z_2, \ldots, Z_n) - f(Z'_1, Z_2, \ldots, Z_n)\right)^2]
\end{align*}
where $f$ is a symmetric function of $n$ independent probability variables $Z_1, Z_2, \ldots, Z_n$, and $Z'_1$ is an i.i.d. copy of $Z_1$.
    Since the random variable $v_i$ could be considered as the sum of $n$ independent Bernoulli random variables of probability $p$, $X_i$ is a symmetric function of $n$ independent Bernoulli random variables of probability $p$.
    In particular, we have $v_i = \sum_{j=1}^n Z_i$.
    We denote $v'_i = \sum_{j \neq j'} Z_j + Z'_{j'}$ and $\omega = \sum_{j \neq j'} Z_j$.
    We prove the following inequality.
    \begin{align}
        \mathbb{E}[(X_i - X'_i)^2] = \nonumber\\
        \mathbb{E}\left[ \left( t_i \left(\ln{\frac{(v_i+1)(1-p)}{(n-v_i+1)p}} - \frac{v_i+1}{np} + \frac{n-v_i+1}{n(1-p)}\right) \right. \right.\nonumber\\
        \left. \left. - t_i \left(\ln{\frac{(v_i+1)(1-p)}{(n-v'_i+1)p}} - \frac{v'_i+1}{np} + \frac{n-v'_i+1}{n(1-p)}\right) \right)^2 \right] \nonumber\\
        \leq 2t_i^2p(1-p)\mathbb{E} \left[ \frac{1}{(\omega + 1)^2} + \frac{1}{(n-\omega)^2} + \frac{2}{(\omega + 1) (n-\omega)} \right. \nonumber\\
        \left. + \frac{1}{n^2p^2(1-p)^2} - \frac{2}{np(1-p)} \left(\frac{1}{\omega+1}\right. \right. \nonumber\\
        \left. \left. -\frac{1}{2(\omega+1)^2(\omega+1)} + \frac{1}{n-\omega} - \frac{1}{2(n-\omega)^2} \right) \right]. \nonumber 
    \end{align}

Applying inequalities \eqref{ieqlt:exp-bin-1}, \eqref{ieqlt:exp-bin-2}, \eqref{ieqlt:exp-bin-3}, \eqref{ieqlt:exp-bin-4}, and \eqref{ieqlt:exp-bin-5}, we get the following inequality
\begin{align}
    \mathbb{E}[(X_i - X'_i)^2] \leq 2t_i^2 \frac{3p^2-3p+1}{n^2(n+1)(n+2)p^3(1-p)^3}\bigg[ 3n+2 \nonumber\\
    + \frac{2}{(1-p)p} \bigg]. \nonumber
\end{align}
Hence, we obtain an upper bound value of $\sigma_i^2$ as follows
\begin{align}
    \sigma_i^2 \leq \frac{t_i^2}{\Pr(A)} \frac{3p^2-3p+1}{n(n+1)(n+2)p^2(1-p)^2}\bigg[ 3n+2 \nonumber\\
    + \frac{2}{(1-p)p} \bigg]. \nonumber
\end{align}
Inequality \eqref{bernstein-cons-lemma} is then proved by applying Bernstein's inequality.

Applying Lemma \ref{basic-lemma} and following the proof of Theorem 1 in \cite{Agarwal2018}, we obtain Theorem \ref{therm-new-estimation-privacy-budget}. \qed

\section{Proof of Proposition \ref{epsilon-properties-remark}}
\label{remark-epsilon-est-pro}
The properties of the privacy budget estimation function in Eq. \eqref{new-privacy-budget-estimation} are proved in appendixes \ref{proof-of-theorem-relationship-P1-P2-models} and \ref{proof-of-replace-p-by-1p-if-plt1o2}.

\section{Proof of Theorem \ref{theorem:1}}
\label{app:c}
Similar to \cite{Agarwal2018}, when the Binomial mechanism and level quantization $\mathcal{M}$ are employed and the learning rate $\gamma$ satisfying $\gamma = \min\{1/L, \sqrt{2G_f} / (\sigma\sqrt{LT})\}$, after SLQBM-FL runs $T$ iterations, we have the following inequality:
	\begin{equation*}
	    \mathbb{E}_{t\sim(\mbox{Unif}[T])} [\Vert \nabla F(\mathbf{w}^t) \Vert^2] \le \frac{2G_fL}{T} + \frac{2\sqrt{2LG_f}}{\sqrt{T}}\sigma + \sqrt{d}GC,
	\end{equation*}
	where ${E}_{t\sim(\mbox{Unif}[T])}[\cdot]$ is the expectation of 2-norm gradient when $t$ is uniformly sampled from $T$ iterations and:
	\begin{align*}
	    \sigma^2 =& \max_{1 \le t \le T} \mathbb{E}\left[\Vert \mathbf{g}(\mathbf{w}^t) - \nabla F(\mathbf{w}^t)\Vert^2 \right] \\
	    &+ \max_{1 \le t \le T} \mathbb{E}_{\mathcal{M}}\left[\Vert \mathbf{g}(\mathbf{w}^t) - \tilde{\mathbf{g}}(\mathbf{w}^t) \Vert^2 \right], \\
	    C =& \max_{1 \le t \le T} \Vert \mathbb{E}_{\mathcal{M}} \left[ \mathbf{g}(\mathbf{w}^t) - \tilde{\mathbf{g}}(\mathbf{w}^t) \right] \Vert.
	\end{align*}
	
Applying the Cauchy–Schwarz inequality, we have: $ \Vert \mathbb{E}_{\mathcal{M}} \left[ \mathbf{g}(\mathbf{w}^t) - \tilde{\mathbf{g}}(\mathbf{w}^t) \right] \Vert^2 \leq \mathbb{E}_{\mathcal{M}}\left[\Vert \mathbf{g}(\mathbf{w}^t) - \tilde{\mathbf{g}}(\mathbf{w}^t) \Vert^2 \right]$.
    Therefore, we get
	\begin{align}
		\mathbb{E}_{t\sim(\mbox{Unif}[T])} [\Vert \nabla F(\mathbf{w}^t) \Vert^2] \le \frac{2G_fL}{T} + \frac{2\sqrt{2LG_f}}{\sqrt{T}}\sigma \nonumber\\
		+ \sqrt{d}G\sqrt{B},
		\label{eq:7}
	\end{align}
	where: 
	\begin{alignat}{3}
		B &= \max_{1 \le t \le T} \mathbb{E}_{\mathcal{M}}\left[\Vert \mathbf{g}(\mathbf{w}^t) - \tilde{\mathbf{g}}(\mathbf{w}^t) \Vert^2 \right], 	\label{eq:8}\\
		U &= \max_{1 \le t \le T} \mathbb{E}\left[\Vert \mathbf{g}(\mathbf{w}^t) - \nabla F(\mathbf{w}^t)\Vert^2 \right],	\label{eq:9}\\
		\sigma^2 &= B + U, \label{eq:10}
	\end{alignat}	 			
	here, $U$ is the variance of the global gradient, and $B$ represents the quadratic bias introduced by $\mathcal{M}$. 	
	If $B = 0$, $\tilde{\mathbf{g}}(\mathbf{w}^t)$ is an unbiased estimation of $\nabla F(\mathbf{w}^t)$ and the SLQBM-FL becomes unbiased with $\Vert \mathbf{g}(\mathbf{w}^t) - \nabla F(\mathbf{w}^t)\Vert^2$ is bounded by $\sigma^2$. Equation \eqref{eq:7} indicates that the algorithm is expected to converge when $T \rightarrow +\infty$.
	 
For the sake of convenience, we denote the following:
	\begin{equation}
		\mathcal{J} = \frac{2G_fL}{T} + \frac{2\sqrt{2LG_f}}{\sqrt{T}}\sigma + \sqrt{d}G\sqrt{B},
	\end{equation}
then from \eqref{eq:7} we get
	\begin{equation}
		\mathbb{E}_{t\sim(\mbox{Unif}[T])} [\Vert \nabla F(\mathbf{w}^t) \Vert^2] \le \mathcal{J}.
		\label{eq:28}
	\end{equation}
On the other hand,  for some $\tau > 0$, Markov's inequality states that
	\begin{equation}
		\Pr\{X \leq \tau\mathbb{E}[X]\} \geq 1- \frac{1}{\tau}.
		\label{eq:29}
	\end{equation}
	 From \eqref{eq:28} and \eqref{eq:29}, we obtain
	\begin{equation}
		\Pr\{\Vert \nabla F(\mathbf{w}^t) \Vert^2 \leq \tau\mathcal{J}\} \geq 1- \frac{1}{\tau},
	\end{equation}
	and citing the proof of B-SGD in \cite{Ghadimi2013}. Theorem 1 is proved.
     \footnote{Following the instructions in \cite{Ghadimi2013}, we can derive the exact number of iterations performed by SLQBM-FL to achieve an $(\theta, \Lambda)$-solution as follows:
            \[\bigg( \frac{LG_f \sigma + \sqrt{L^2G_f^2\sigma^2 + \theta\Lambda L^2G_f^2}}{\theta\Lambda}\bigg)^2,\]
    which is simplified as:
    \[\mathcal{O}\bigg(\frac{1}{\theta\Lambda} + \frac{\sigma^2}{\theta^2\Lambda^2}\bigg).\]
            }\qed
	

\section{Proof of Theorem \ref{upperbound-convergence-rate}}\label{U-B-upperbound-proof}
We denote the elements of vectors $\mathbf{g}_k(\cdot)$ and $\nabla f_k(\cdot)$ as $g^j_k(\cdot)$ and $\nabla f^j_k(\cdot)$, for $1 \leq j \leq d$, respectively.
    We have: $\mathbf{g}_k(\mathbf{w}^t$) = $\nabla f_k(\mathbf{w}^t)$ for $k \in \mathcal{K}$.

We have:
\begin{align*}
    U &=& \max_{1 \leq t \leq T} \mathbb{E}\big[\big|\big| \mathbf{g}(\mathbf{w}^t)-\nabla F(\mathbf{w}^t)\big|\big|^2\big] \\
    &=& \max_{1 \leq t \leq T} \mathbb{E}\bigg[\bigg|\bigg|\frac{1}{K} \sum_{k\in\mathcal{K}} \mathbf{g}_k(\mathbf{w}^t) - \frac{1}{M}\sum_{k=1}^{M} \nabla f_k(\mathbf{w}^t)\bigg|\bigg|^2\bigg]
\end{align*}
\begin{align*}
    &=& \frac{1}{K^2M^2}\max_{1 \leq t \leq T} \mathbb{E}\bigg[\bigg|\bigg|M \sum_{k\in\mathcal{K}} \mathbf{g}_k(\mathbf{w}^t) \\
    &&- K\sum_{k=1}^{M} \nabla f_k(\mathbf{w}^t)\bigg|\bigg|^2\bigg]
\end{align*}
\begin{align*}
    &=& \frac{1}{K^2M^2}\max_{1 \leq t \leq T} \mathbb{E}\bigg[\bigg|\bigg|(M-K) \sum_{k\in\mathcal{K}} \mathbf{g}_k(\mathbf{w}^t) \\
    &&- K\sum_{k=1, \mbox{ }  
k \notin \mathcal{K}}^{M} \nabla f_k(\mathbf{w}^t)\bigg|\bigg|^2\bigg] \\
\end{align*}
\begin{align*}
    &=& \frac{1}{K^2M^2} \sum_{j=1}^{d}\max_{1 \leq t \leq T} \mathbb{E}\bigg[\bigg((M-K) \sum_{k\in\mathcal{K}} g^j_k(\mathbf{w}^t) \\
    &&- K\sum_{k=1, \mbox{ }  
k \notin \mathcal{K}}^{M} \nabla f^j_k(\mathbf{w}^t)\bigg)^2\bigg] \\
   &\leq& \frac{1}{K^2M^2} \sum_{j=1}^{d}\max_{1 \leq t \leq T} \mathbb{E}\bigg[\bigg((M-K) \sum_{k\in\mathcal{K}} \big| g^j_k(\mathbf{w}^t)\big| \\
    &&+ K\sum_{k=1, \mbox{ } 
k \notin \mathcal{K}}^{M} \big| \nabla f^j_k(\mathbf{w}^t)\big|\bigg)^2\bigg]
\end{align*}
\begin{align*}
&\leq&\frac{d}{K^2M^2}\bigg[ (M-K)KG + (M-K)KG \bigg]^2\\
&=& \frac{4d(M-K)^2G^2}{M^2}.
\end{align*}

Inequality \eqref{U-upperbound} is proved.
    Let's consider the case that the means of each gradient element of all devices are identical, i.e., $\mathbb{E}[\nabla f^j_k({w}^t)] = \mathbb{E}[\nabla f^{j}_{k'}({w}^t)] = F_j$, for $1 \leq k \neq k' \leq M$ and $\nabla f^j_k(\cdot)$ are independent of each other over all devices.
    In practice, the gradient values of selected devices are aggregated. Therefore, they are not independent, i.e., the correlation between the gradient elements $\mbox{Corr}(\nabla f_k^j(\mathbf{w}^t), f_{k'}^j(\mathbf{w}^t)) \neq 0$.
    However, we quantize and add random noise to each gradient element.
    Moreover, in the Federated learning framework where $K$ devices are chosen randomly and the total number of devices is very large compared to the chosen number of devices, i.e., $M >> K$, the correlation between $\nabla f_k^j(\mathbf{w}^t$) and $\nabla f_{k'}^j(\mathbf{w}^t$) is small in comparison to $G$, for $k \neq k'$.
    Thus, the case where the gradient elements are i.i.d. is worth considering.
\begin{align*}
    U &=& \frac{1}{K^2M^2} \sum_{j=1}^{d}\max_{1 \leq t \leq T} \mathbb{E}\bigg[\bigg((M-K) \sum_{k\in\mathcal{K}} \big(g^j_k(\mathbf{w}^t) \\
    &&- F_j\big) - K\sum_{k=1, \mbox{ } 
k \notin \mathcal{K}}^{M} \big(\nabla f^j_k(\mathbf{w}^t) - F_j \big)\bigg)^2\bigg] \\
    &\overset{(a)}{=}& \frac{1}{K^2M^2} \sum_{j=1}^{d}\max_{1 \leq t \leq T} \bigg[(M-K) \sum_{k\in\mathcal{K}} \mbox{Var}\big(g^j_k(\mathbf{w}^t)\big) \\
    &&+ K\sum_{k=1,\mbox{ }  
k \notin \mathcal{K}}^{M} \mbox{Var}\big(\nabla f^j_k(\mathbf{w}^t)\big)\bigg] \\
    &\overset{(b)}{\leq}& \frac{d}{K^2M^2} 2K(M-K)4G^2 = \frac{8(M-K)}{M^2} \frac{dG^2}{K},
\end{align*}
where $(a)$ uses the facts that $g^j_k(\mathbf{w}^t)$ and $g^j_{k'}(\mathbf{w}^t)$ are independent for $k \neq k'$, and $\mathbb{E}[g^j_k(\mathbf{w}^t) - F_j] = 0$. $(b)$ uses the fact that $\mbox{Var}(\nabla f_k^j(\mathbf{w}^t)) \leq 4G^2$ since $|\nabla f_k^j(\mathbf{w}^t)| < G$. We derive Inequality \eqref{assumped-U-upperbound}.

Inequality \eqref{B-upperbound} is proved similarly to the proof of Theorem 3 in \cite{Agarwal2018}.
Recall that, $\tilde{g}^j_k(\mathbf{w}^t) = s \big(\big\lfloor g^j_k(\mathbf{w}^t) / s \big\rfloor + \delta^j_k + z^j_k - np\big)$, where:
\begin{align*}
    s = \frac{2G}{l-1},
\end{align*}
\begin{equation*}
    \delta_k^j =
	\begin{cases}
	1 & \mbox{with probability } 
 \frac{g_k^j(\mathbf{w}^t)}{s} 
 -  \big\lfloor\frac{g_k^j(\mathbf{w}^t)}{s}\big\rfloor, \\
	0 & \mbox{otherwise},
	\end{cases}	
\end{equation*}
\begin{equation*}
    z^j_k \sim \mathcal{B}(n, p),
\end{equation*}
we have:
\begin{align*}
    B &=& \max_{1 \leq t \leq T} \mathbb{E} \big[ \big|\big| \tilde{\mathbf{g}}(\mathbf{w}^t) - \mathbf{g}(\mathbf{w}^t)\big|\big|^2\big]\\
    &=& \qquad\space\space \frac{1}{K^2} \max_{1 \leq t \leq T} \mathbb{E} \bigg[ \bigg|\bigg| \sum_{k\in\mathcal{K}} \big(\tilde{\mathbf{g}}_k(\mathbf{w}^t) - \mathbf{g}_k(\mathbf{w}^t)\big) \bigg|\bigg|^2\bigg]
\end{align*}
\begin{align*}
    &=& \frac{1}{K^2} \sum_{j=1}^d \max_{1 \leq t \leq T} \mathbb{E} \bigg[ \bigg( \sum_{k\in\mathcal{K}} \big(\tilde{g}^j_k(\mathbf{w}^t) - g^j_k(\mathbf{w}^t)\big) \bigg)^2\bigg]
\end{align*}
\begin{align*}
    &=& \frac{1}{K^2} \sum_{j=1}^d \max_{1\leq t \leq T} \mathbb{E} \bigg[ s^2 \bigg( \sum_{k\in\mathcal{K}} \bigg( \delta^j_k + \bigg\lfloor \frac{g^j_k(\mathbf{w}^t)}{s} \bigg\rfloor - \frac{g^j_k(\mathbf{w}^t)}{s}
\end{align*}
\begin{align*}
    && + z^j_k - np \bigg) \bigg)^2 \bigg]
\end{align*}
\begin{align*}
    &=& \frac{s^2}{K^2} \sum_{j=1}^d \max_{1 \leq t \leq T} \mathbb{E} \bigg[ \bigg( \sum_{k\in\mathcal{K}} \bigg( \delta^j_k + \bigg\lfloor \frac{g^j_k(\mathbf{w}^t)}{s} \bigg\rfloor - \frac{g^j_k(\mathbf{w}^t)}{s} \bigg)
\end{align*}
\begin{align*}
    && + \sum_{k \in \mathcal{K}} \big(z^j_k - np \big) \bigg)^2 \bigg] \\
    &\overset{(c)}{=}& \frac{s^2}{K^2} \sum_{j=1}^d \max_{1 \leq t \leq T} \bigg\{ \mathbb{E} \bigg[ \bigg( \sum_{k\in\mathcal{K}}  \bigg( \delta^j_k + \bigg\lfloor \frac{g^j_k(\mathbf{w}^t)}{s} \bigg\rfloor   
\end{align*}
\begin{align*}
    && - \frac{g^j_k(\mathbf{w}^t)}{s} \bigg)\bigg)^2 \bigg] + \mathbb{E} \bigg[ \sum_{k \in \mathcal{K}} \bigg( z^j_k - np \bigg)^2 \bigg] \bigg\} \\
    &=& \frac{s^2}{K^2} \sum_{j=1}^d \max_{1 \leq t \leq T}  \mathbb{E} \bigg[ \bigg( \sum_{k\in\mathcal{K}} \bigg( \delta^j_k + \bigg\lfloor \frac{g^j_k(\mathbf{w}^t)}{s} \bigg\rfloor  \\
    && - \frac{g^j_k(\mathbf{w}^t)}{s} \bigg) \bigg)^2 \bigg] + \frac{s^2}{K^2} \sum_{j=1}^d \sum_{k\in\mathcal{K}} \mbox{Var}(\mathcal{B}(n, p)) \\
    &=& \frac{s^2}{K^2} \sum_{j=1}^d \max_{1 \leq t \leq T} \mathbb{E} \bigg[ \bigg( \sum_{k\in\mathcal{K}} \bigg( \delta^j_k + \bigg\lfloor \frac{g^j_k(\mathbf{w}^t)}{s} \bigg\rfloor  \\
    &&  - \frac{g^j_k(\mathbf{w}^t)}{s} \bigg) \bigg)^2 \bigg] + \frac{4G^2dnp(1-p)}{K(q-1)^2},
\end{align*}
where $(c)$ uses the facts that $\delta^j_k + \big\lfloor g^j_k(\mathbf{w}^t)/s \big\rfloor - g^j_k(\mathbf{w}^t)/s$ and $z^j_k - np$ are independent, $z^j_k - np$ and $z^j_{k'} - np$ are independent for $k \neq k'$, and $\mathbb{E} \big[\delta^j_k + \big\lfloor g^j_k(\mathbf{w}^t)/s \big\rfloor - g^j_k(\mathbf{w}^t)/s\big] = \mathbb{E}\big[z^j_k - np\big] = 0$. On the other hand, since: $0 \leq \big(\delta_k^j + \big\lfloor g^j_k(\mathbf{w}^t)/s \big\rfloor| - g^j_k(\mathbf{w}^t)/s 
\big) \leq 1$, we have: 
$$0 \leq \mathbb{E} \bigg[ \bigg( \sum_{k \in \mathcal{K}}\bigg( \delta^j_k + \bigg\lfloor \frac{g^j_k(\mathbf{w}^t)}{s} \bigg\rfloor - \frac{g^j_k(\mathbf{w}^t)}{s} \bigg) \bigg)^2 \bigg] \leq K.$$
Consequently, we obtain the following result:
\begin{align*}
    \frac{4G^2dnp(1-p)}{K(q-1)^2} \leq B \leq \frac{4G^2d\big(1+np(1-p)\big)}{K(q-1)^2}.
\end{align*}
Theorem \ref{upperbound-convergence-rate} is proved.
\qed

\section{Proof of Theorem \ref{relationship-P1-P2-models}}
\label{proof-of-theorem-relationship-P1-P2-models}
At first, we state Lemma \ref{theorem:state-of-value-m} to aid the proof of Theorem \ref{relationship-P1-P2-models}.
    \begin{lemma} [The domain range of the Binomial trial number $n$ in dependence to the level quantization $q$, Binomial distribution's parameter $p$ and privacy budget upper bound $\bar{\epsilon}$]
    \label{theorem:state-of-value-m}
    	    For each upper bound $\bar{\epsilon}$ of privacy budget, fixing values of $q$ and $p$ there exists an integer $n_1$ such that $\epsilon(n) \leq \bar{\epsilon}$ if and only if $n \geq n_1$.
	\end{lemma}
	\begin{proof}
        We transform some terms of the right-hand side of Eq. \eqref{new-privacy-budget-estimation} as follows,
        \begin{align*}
            \frac{np(1-p)+1}{n^2p^2(1-p)^2} = \frac{1}{np(1-p)} + \frac{1}{n^2p^2(1-p)^2},
        \end{align*}
        \begin{align*}
            S_1 &=& \frac{3(3p^2-3p+1)}{(n+1)(n+2)p^2(1-p)^2}\\
            &&+ \frac{(3p^2-3p+1)\left( 2 + \frac{2}{p(1-p)} \right)}{n(n+1)(n+2)p^2(1-p)^2},
        \end{align*}
        \begin{align*}
            \frac{S_2}{n^2p^2(1-p)^2} &=& \frac{\sqrt{2\ln{\frac{20d}{\delta}}}}{n^{\frac{3}{2}}p^{\frac{3}{2}}(1-p)^{\frac{3}{2}}}\\
            &&+ \frac{1 + \frac{2}{3}\max\{p, 1-p\}}{n^2p^2(1-p)^2}.
        \end{align*}
        Replace the above equations into Eq. \eqref{new-privacy-budget-estimation}, it is clear that all terms monotonically decrease as $n$ increases by fixing $p$ and $q$.
        Lemma \ref{theorem:state-of-value-m} is proved.
        \end{proof}
    
Let's consider an arbitrary feasible solution $\mathcal{S}_2 = (q_2, n_2, p_2, P_{k2})$ of $(\Phi_2)$.
    Based on constraint \eqref{eq:value_m}, we get $n_2 \geq \lfloor \max\{23 \ln(10d / \delta), 2(q_2 + 1)\} / (Kp_2(1 - p_2)) \rfloor$ that leads to $Kn_2p_2(1 - p_2) \geq \max\{23 \ln(10d / \delta), 2(q_2 + 1)\}$ according to the constraint \eqref{eq:DP_cons}.
    It is clear that constraint \eqref{eq:P_k_compute_cons} leads to $P_{k2} \geq \omega_0[(q_2 + n_2)^{d / (TW)} - 1] / h_k$ that is equivalently to $d \log_2(q_2 + n_2) \leq TW \log_2(1 + P_{k2}h_k / \omega_0)$.
    It proves that solution $\mathcal{S}_2$ satisfies the constraint \eqref{eq:channel_cap_cons}.
    The constraints \eqref{eq:P_k_compute_cons} implies that $P_k^{\min} \leq P_{k2} \leq P_k^{\max}$. Thus, the constraints \eqref{eq:trans_pow_cons} are satisfied. Therefore,
    $\mathcal{S}_2$ satisfies all the constraints of ($\Phi_1$) then it is also a feasible solution of ($\Phi_1$).
	
Next, we will prove that any optimal solution of $(\Phi_2)$ is also an optimal solution of $(\Phi_1)$.
    We consider an arbitrary optimal solution $\mathcal{S}^* = (q^*, n^*, p^*, P_{k}^*)$ of $(\Phi_1)$.
    Since $S^*$ is an optimal solution of ($\Phi_1$) and $\varphi$ is an increasing function of $n$ when fixing $q$ and $p$, $\mathcal{S}^*$ satisfies the constraint \eqref{eq:value_m}.
    Consider $P_k^{*}$, we have can see that $P_k^{*} \geq P_k^{\min}$ and $P_k^{*} \geq \omega_0 [(q^{*} + n^{*})^{d/(TW)} - 1] / h_k$.
    Therefore, $P_k^{\min} \leq \max\{P_k^{\min}, \omega_0[(q^* + n^*)^{d/(TW)} - 1] / h_k\} \leq P_k^* \leq P_k^{\max}$.
    We consider $P_k' = \max\{P_k^{\min}, \omega_0[(q^* + n^*)^{d/(TW)} - 1] / h_k\}$.
    It is clear that  $\mathcal{S}' = (q^*, n^*, p^*, P_k')$ is a feasible solution of ($\Phi_2$).
    On the other hand, since $\varphi(\mathcal{S}') = \varphi(\mathcal{S}^*)$, the optimal objective function value of ($\Phi_1$) is not less than the optimal objective function value of ($\Phi_2$).
    But as proved above, the feasible solution set of ($\Phi_2$) is a subset of the feasible solution set of ($\Phi_1$).
    Thus, $\mathcal{S}'$ is also an optimal solution of ($\Phi_2$). Therefore, any optimal solution of ($\Phi_2$) is also an optimal solution of ($\Phi_1$).
    Statement \textit{(i)} of Theorem \ref{relationship-P1-P2-models} is proved.
	
As proved above, considering any optimal solution $S^*$ of ($\Phi_1$) we can compute an optimal solution $S'$ of ($\Phi_2$).
    Therefore, if ($\Phi_1$) is feasible then ($\Phi_2$) is also feasible.
    In addition, by applying proof by contradiction, we get that if ($\Phi_2$) is infeasible then ($\Phi_1$) is also infeasible. Statement \textit{(ii)} of Theorem \ref{relationship-P1-P2-models} is proved.

Now, we have the following observations:
\begin{itemize}
	    \item If $(q^*, n^*, p^*, P_k^*)$ is an optimal solution of ($\Phi_2$), $(q^*, n^*, p^*, \tilde{P}_k)$ is also optimal solution of ($\Phi_1$), for $P_k^* \leq \tilde{P}_k \leq P_k^{\max}$.
	    \item If $(q^*, n^*, p^*, P_k^*)$ is an optimal solution of ($\Phi_1$), $(q^*, n^*, p^*, \tilde{P}_k)$ is also optimal solution of ($\Phi_2$), where $\tilde{P}_k = \max \{P_k^{\min}, \omega_0 [ (q^* + n^*)^{d/(TW)} - 1 ]/h_k\}$.
	\end{itemize}

    The above observations about the optimal solutions of ($\Phi_1$) and ($\Phi_2$) show that after solving the optimal solutions of ($\Phi_2$) and then varying the value of the transmit power of these solutions, we can derive all the optimal solutions of ($\Phi_1$).
    This proves the statement (\textit{iii}). \qed
	
\section{Proof of Lemma \ref{thr:replace-p-by-1p-if-plt1o2}}
\label{proof-of-replace-p-by-1p-if-plt1o2}
First, since $3p^2 - 3p + 1 = 3(1-p)^2 - 3(1-p) + 1$, we have $\epsilon(q, n, p) = \epsilon(q, n, 1-p)$.
    Second, it is clear that $\varphi(q, n, p, P_k) = \varphi(q, n, 1 - p, P_k)$.
    Third, formulas of the constraints \eqref{eq:DP_cons} and \eqref{eq:value_m} which contain $p$ do not change if we replace $p$ by $1 - p$.
    Therefore, we conclude that if $(\tilde{q}, \tilde{n}, \tilde{p}, \tilde{P}_k)$ is a feasible solution of problem ($\Phi_1$) then $(\tilde{q}, \tilde{n}, 1 - \tilde{p}, \tilde{P}_k)$ is also a feasible solution of problem ($\Phi_1$) with the equal objective value.
    The same conclusion holds for ($\Phi_2$).
    As a result, we only need to consider $p \geq 1/2$. \qed
	
\section{Proof of Lemma \ref{thr:narrow-quatization-level-1}}
\label{proof-of-narrow-quatization-level-1}
We replace $2D / s$ with  $q - 1$ in Eqs. \eqref{eq:Delta_1-1-exp}, \eqref{eq:Delta_2-1-exp} and \eqref{eq:Delta_inf-1-exp} and obtain:
	\begin{alignat}{2}
	&\Delta_1  &&= \sqrt{d}(q - 1) + \sqrt{2\sqrt{d}(q - 1)\ln\frac{2}{\delta}} + \frac{4}{3}\ln\frac{2}{\delta}, \nonumber\\
	&\Delta_2  &&= q - 1 + \sqrt{\Delta_1 +  \sqrt{2\sqrt{d}(q - 1)\ln\frac{2}{\delta}}}, \nonumber\\
	&\Delta_\infty &&= q + 1. \nonumber
	\end{alignat}
	
We have: $np(1 - p) \leq n/4 = (q + n - q) / 4$.
    Combining this inequality and constraints \eqref{eq:channel_cap_cons}, \eqref{eq:trans_pow_cons}, we obtain:
	\begin{align}
	    np(1 - p) \leq \frac{1}{4} \left[ \left(1 + \min_{k \in \mathcal{K}} \frac{P_k^{\max} h_k}{\omega_0}\right)^{\frac{TW}{d}} - q \right] = r(q).\nonumber
	\end{align}

We provide bounds on $S_1$ and $S_2$ in Eqs. \eqref{s-1-formula} and \eqref{s-2-formula} as follows.
    \begin{align}
        S_1 &\geq \frac{3p^2-3p+1}{n^3p^2(1-p)^2}\left[2n + \frac{2}{p(1-p)}\right]\nonumber\\
        &\geq \frac{\frac{1}{4}}{n^3p^2(1-p)^2}\left[2n + \frac{2}{p(1-p)}\right] \nonumber\\
        &\geq \frac{r(q)+1}{2r(q)^3}, \nonumber\\
        S_2 &\geq \left( \sqrt{2np(1-p)\ln{\frac{20d}{\delta}}} + 1 + \frac{2}{3} \frac{1}{2}     \ln{\frac{20d}{\delta}} \right) ^2 \nonumber\\
        \nonumber
    \end{align}

Applying the above inequalities,  we derive the following lower bounds of the terms in the right-hand side of Eq. \eqref{new-privacy-budget-estimation}
\begin{align}
    \frac{\Delta_2\sqrt{2\ln \frac{1.25}{\delta}}}{\sqrt{np(1-p)}} \geq \frac{\Delta_2\sqrt{2\ln \frac{1.25}{\delta}}}{\sqrt{r(q)}} = g_1(q), \\
    \frac{\alpha \Delta_1 \left( np(1-p) + 1 \right)}{n^2p^2(1-p)^2\left(1-\frac{\delta}{10}\right)} (p^2 + (1-p)^2) \geq \nonumber
\end{align}
\begin{align}
    \frac{\alpha \Delta_1 (r(q) + 1)}{2\left(1 - \frac{\delta}{10}\right)r^2(q)} = g_2(q), \\
    \frac{\Delta_2}{\sqrt{1 - \frac{\delta}{10}}} \sqrt{2S_1 \ln \frac{10}{\delta}} \geq \frac{\Delta_2}{\sqrt{1 - \frac{\delta}{10}}} \nonumber\\ 
    \sqrt{\frac{r(q)+1}{r^3(q)}\ln \frac{10}{\delta}} = g_3(q),\\
    \frac{2}{3}\alpha \frac{S_2 \left(p^2 + (1-p)^2\right) \ln \frac{10}{\delta} \Delta_{\infty}}{n^2p^2(1-p)^2} \geq \frac{\alpha}{3} \ln{\frac{10}{\delta}} \Delta_{\infty} \nonumber\\
     \left[ \sqrt{\frac{2\ln{\frac{20d}{\delta}}}{r(q)}} + \frac{3 + \ln{\frac{20d}{\delta}}}{3r(q)}\right]^2 = g_4(q),\\
    \frac{2\ln{\frac{1.25}{\delta}}\Delta_{\infty}}{np(1-p)} \geq \frac{2\ln{\frac{1.25}{\delta}}\Delta_{\infty}}{r(q)} = g_5(q).
\end{align}
Therefore, we get
\begin{align}
    \epsilon \geq g_1(q) + g_2(q) + g_3(q) + g_4(q) + g_5(q) = g(q).
    \label{upper-bound-level-quantization-function}
\end{align}
Since $g_i(q)$ monotonically increases with respect to $q$, for $i = 1, 2, \ldots, 5$, $g(q)$ monotonically increases with respect to $q$. Therefore, for each value $\bar{\epsilon}$, we can compute an integer $\bar{q}$ such that $g(q) \leq \bar{\epsilon}$ for every $q \leq \bar{q}$. Lemma \ref{thr:narrow-quatization-level-1} is proved. \qed

\section{Proof of Theorem \ref{thr:relative-error-statement}}
\label{proof-theorem-relative-error-statement}
The expressions which contain $p$ and appear in the terminators of the terms in Eq. \eqref{new-privacy-budget-estimation} are as follows:
\begin{align*}
    p^2 + (1-p)^2,\\
    \frac{2}{3}\max\{p, 1-p\},\\
    3p^2-3p+1.
\end{align*}
    These expressions monotonically increase with respect to $p$ for $p \in [1/2, 1)$.
    Conversely, the expressions which contain $p$ and appear in the denominators of the terms in Eq. \eqref{new-privacy-budget-estimation} are the powers of $p(1-p)$ with non-negative exponents, where $p(1-p)$ monotonically decreases with respect to $p$ for $p \in [1/2, 1)$.
    Therefore, $\epsilon$ monotonically increases with respect to $p$ for $p \in [1/2, 1)$.
	
We consider an arbitrary optimal solution $(q^*, n^*, p^*, P_k^*)$ of $(\Phi_2)$.
In the case: $p^* = 1/2$, since $1/2$ is an element of the search domain of $\mathcal{P}$ of Algorithm \ref{alg:dis-general-solution}, it is clear that the solution returned by Algorithm \ref{alg:dis-general-solution} is $(q^*, n^*, p^*, P_k^*)$.
Therefore, we have $\tilde{\varphi} = \varphi^*$, that leads to $\tilde{\varphi} / \varphi^* = 1 < 1 + \rho$.
We now consider $p^* > 1/2$.
With fixed variables $q = q^*$, $n = n^*$, $P_k = P_k^*$, $\varphi(q^*, n^*, p, P^*_k)$ is a continuous function over $p \in [1/2, 1)$.
    Therefore, there exists a closed interval $[p^* - \lambda_1, \ p^* + \lambda_1] \subset [1/2, 1)$ such that $\varphi(q^*, n^*, p, P^*_k) / \varphi^* < 1 + \rho$ for all $p \in [p^* - \lambda_1, p^* + \lambda_1]$.
	
Fixing variables $q = q^*$, $n = n^*$, we consider $\epsilon$ as mono-variable function with respect to $p \in [1/2, 1)$.
Since $\epsilon$ monotonically increases with respect to $p$ for $p \in [1/2, 1)$, if $\epsilon(p^*)$ is a local minimum point,  we have $p^* = 1/2$.
We already proved that if $p^* = 1/2$, we have $\tilde{\varphi} / \varphi^* < 1 + \rho$.
    Therefore, we only need to consider the case where $\epsilon(p^*)$ is not a local minimum point of $\epsilon(p)$ and $p^* > 1/2$.
	
Note that $\epsilon(q^*, n^*, p)$ is a continuous function with respect to $p \in [1/2, 1)$.
    Since $\epsilon(p^*)$ is not a local minimum point, there exists $\lambda_2$ such that for all $p \in [p^* - \lambda_2, \ p^*] \subset [1/2, 1)$ or for all $p \in [p^*, p^* + \lambda_2] \subset [1/2, 1)$, we have $\epsilon(q^*, n^*, p) < \epsilon(q^*, n^*, p^*) \leq \bar{\epsilon}$.
	
Let's denote $\lambda_3 = \min\{\lambda_1, \lambda_2\}$.
    There exists a positive real $\bar{\lambda} \leq \lambda_3$ such that for any $\lambda \leq \bar{\lambda}$, we have that $\{i\lambda| i \in \mathbb{N}^+, 1/2 < i\lambda < 1
	\} \cap [p^* - \lambda_3, p^*]$ is not empty.
	Let's consider an arbitrary element $p' \in \{i\lambda | i \in \mathbb{N}, 1/2 \leq i\lambda < 1\} \cap [p^* - \lambda_3, p^*]$ for an arbitrary $\lambda \leq \bar{\lambda}$.
	
Algorithm \ref{alg:m-epsilon-domain} with parameters $q = q^*, p = p'$ returns $n_1'$ such that $\epsilon(q^*, p', n) \leq \bar{\epsilon}$ if and only if $n \geq n_1'$. Since $p' \in [p^* - \lambda_3, p^*]$, we have $\epsilon(q^*, p', n^*) \leq \bar{\epsilon}$. Therefore, we have $n^* \geq n_1'$.
    Let's denote:
	\begin{align}
	    n' = \max \bigg\{ \left\lceil \frac{\max\{23 \ln\frac{10d}{\delta}, \ 2(q^* + 1)\}} {Kp'(1- p')} \right\rceil, n_1' \bigg\}.\label{proof-theorem-relative-error-statement-m-strick}
	\end{align}
$(q^*, n', p')$ is a feasible solution of $(\Phi_2)$ and is considered by Algorithm \ref{alg:dis-general-solution}, i.e., $(q^*, p') \in \mathcal{Q}\times\mathcal{P}$.

It is clear that $p(1 - p)$ is a parabolic curve over $p \in (0, 1)$ with the maximum point at $p = 1/2$, and $1/2 < p' < p^*$.
Therefore, $p'(1 - p') > p^*(1 - p^*)$.
We have the following inequality:
	$$\frac{\max\{23 \ln{\frac{10d}{\delta}}, 2(q^* + 1)\}}{Kp'(1 - p')} <  \frac{\max\{23 \ln{\frac{10d}{\delta}}, 2(q^* + 1)\}}{Kp^*(1 - p^*)}.$$
	Therefore, we have:
	\begin{align}
	    n^* \geq \max \bigg\{ \left\lceil \frac{\max\{23 \ln\frac{10d}{\delta}, \ 2(q^* + 1)\}} {Kp'(1- p')} \right\rceil, n_1'\bigg\}.\label{proof-theorem-relative-error-statement-m-star}
	\end{align}
	
Combining \eqref{proof-theorem-relative-error-statement-m-strick} and \eqref{proof-theorem-relative-error-statement-m-star}, we get $n^* \geq n'$.
    Therefore, $\varphi(q^*, n', p', P_k^*) < \varphi(q^*, n^*, p', P_k^*)$.
    Note that $\varphi(q, n, p, P_k) = (1+np(1-p))/(q-1)^2$ monotonically increases with respect to $n$.
    Moreover, since $p' \in [p^* - \lambda_1, p^* + \lambda_1]$, we have:
	$$\frac{\varphi(q^*, n^*, p', P_k^*)}{\varphi^*} < 1 + \rho.$$
	Consequently, we have:
	$$\frac{\varphi(q^*, n', p', P_k^*)}{\varphi^*} < 1 + \rho.$$
	
Since $(q^*, p') \in \mathcal{Q} \times \mathcal{P}$, the objective function value of the solution returned by Algorithm \ref{alg:dis-general-solution} does not exceed $\varphi(q^*, n', p', P_k^*)$.
    Therefore, the solution returned by Algorithm \ref{alg:dis-general-solution} is a $\rho$-relative solution.
    Recall that the optimal objective values of problems ($\Phi_1$) and ($\Phi_2$) are equal and each feasible solution of ($\Phi_2$) is also a feasible solution of ($\Phi_1$).
    Therefore, the solution returned by Algorithm \ref{alg:dis-general-solution} is also a $\rho$-relative error solution of ($\Phi_1$).
    Theorem \ref{thr:relative-error-statement} is proved. \qed
    
\section{Proof of Theorem \ref{thr:relative-error-estimation}}
\label{proof-thr:relative-error-estimation}
We will prove that:
    \begin{align}
        \frac{\varphi(q^*, n^*, p, P_k^*)}{\varphi(q^*, n^*, p^*, P_k^*)} < 1 + \mu\lambda, \label{proof-relative-fraction}
    \end{align}
    where $p \in (0, 1)$ is an integer multiple of $\lambda$ and $|p - p^*| < \lambda$.   
We transform \eqref{proof-relative-fraction} as follows:
    \begin{align}
        \eqref{proof-relative-fraction} & \Leftrightarrow \frac{1 + n^*p(1 - p)}{1 + n^*p^*(1 - p^*)} < 1 + \mu\lambda \nonumber \\
        & \Leftrightarrow  n^*p(1 - p) < n^*p^*(1 - p^*) + \mu\lambda + \nonumber\\
        & \quad\quad \mu \lambda n^*p^*(1 - p^*) \nonumber \\
        & \Leftrightarrow n^*(p - p^*)(1 - p - p^*) < \mu\lambda n^* p^*(1 - p^*) \nonumber \\
       & \quad\quad + \mu\lambda. \label{relative-inquality-form-1}
    \end{align}
    
We will prove that: $n^*(p - p^*)(1 - p - p^*) < \mu\lambda n^* p^*(1 - p^*)$ or $(p - p^*)(1 - p - p^*) < \mu\lambda p^*(1 - p^*)$. We need to consider only the case where $(p - p^*)(1 - p - p^*) > 0$.
    The case where $(p - p^*)(1 - p - p^*) \leq 0$ is trivial.
    Firstly, we have $\lambda > |p - p^*|$. Secondly, we will prove that: 
    \begin{align}
        \mu p^*(1 - p^*) > |1 - p - p^*|. \label{relative-inquality-form-2}
    \end{align}
    
Considering $(p - p^*)(1 - p - p^*) > 0$, we have two following cases.
    
\textit{Case 1}: $p > p^*$, we have $0 < 1 - p - p^* < 1 - 2p^*$. We will prove that $\mu p^*(1 - p^*) > 1 - 2p^*$, which is equivalent to the quadratic inequality $\mu(p^*)^2 - (\mu + 2)p^* + 1 < 0$.
    The quadratic inequality holds if $(\mu + 2 - \sqrt{\mu^2 + 4}) / (2\mu) < p^* < (\mu + 2 + \sqrt{\mu^2 + 4}) / (2\mu)$. It is clear that, since $\mu > 0$, we have $(\mu + 2 + \sqrt{\mu^2 + 4}) / (2\mu) > 1 > p^*$.
    We transform $(\mu + 2 - \sqrt{\mu^2 + 4}) / (2\mu)$ as follows:
    \begin{align*}
        \frac{\mu + 2 - \sqrt{\mu^2 + 4}}{2\mu} = \frac{2}{\mu + 2 + \sqrt{\mu^2 + 4}} < \frac{1}{\mu} \nonumber \\
        = \frac{1 - \sqrt{1 - 4\eta}}{2}.
    \end{align*}
    
Now, we need to prove that: 
    \begin{align}
        p^* \geq \frac{1 - \sqrt{1 - 4\eta}}{2}. \label{relative-inquality-form-3}
    \end{align}
    
\textit{Case 2}: $p < p^*$, then $0 < p + p^* - 1 < 2p^* - 1$.
    We will prove that: $\mu p^*(1 - p^*) > 2p^* - 1$, which is equivalent to the quadratic inequality $\mu(p^*)^2 - (\mu - 2)p^* - 1 < 0$.
    The quadratic inequality holds if $(\mu - 2 - \sqrt{\mu^2 + 4}) / (2\mu) < p^* < (\mu - 2 + \sqrt{\mu^2 + 4}) / (2\mu)$.
    It is clear, since $\mu > 0$, we have $(\mu - 2 - \sqrt{\mu^2 + 4}) / (2\mu) < 0 < p^*$. We transform $(\mu - 2 + \sqrt{\mu^2 + 4}) / (2\mu)$ as follows:
    \begin{align*}
        \frac{\mu - 2 + \sqrt{\mu^2 + 4}}{2\mu} > \frac{2\mu - 2}{2\mu} &= 1 - \frac{1}{\mu} = 1 \nonumber \\ - \frac{1 - \sqrt{1 - 4\eta}}{2} 
        &= \frac{1 + \sqrt{1 - 4\eta}}{2}.
    \end{align*}
    
Now, we need to prove that: 
    \begin{align}
        \frac{1 + \sqrt{1 - 4\eta}}{2} \geq p^*. \label{relative-inquality-form-4}
    \end{align}
    
Applying \eqref{eq:DP_cons}, we have: $p^*(1 - p^*) \geq \max\{23\ln(10d / \delta), 2(q^* + 1)\} / (Kn^*) \geq \max\{23\ln(10d / \delta), 6\} / (K\bar{n}_{\mathcal{N}}) = \eta$, then $0 \geq (p^*)^2 - p^* + \eta$. We get $(1 - \sqrt{1 - 4\eta}) / 2 \leq p^* \leq (1 + \sqrt{1 - 4\eta}) / 2$.
    Therefore, the inequalities \eqref{relative-inquality-form-3} and \eqref{relative-inquality-form-4} hold.
    Hence, the inequality \eqref{relative-inquality-form-1} holds, and Theorem \ref{thr:relative-error-estimation} is proved. \qed
	
	

\ifCLASSOPTIONcaptionsoff
\newpage
\fi
	
\bibliography{ref}
\bibliographystyle{IEEEtran}

\end{document}